\documentclass[twoside,11pt]{article}

\usepackage{jmlr2e}

\usepackage{amsmath,amsfonts,amscd}
\usepackage{bbm}

\usepackage{hyperref}
\usepackage{algorithm}
\usepackage{algpseudocode}
\usepackage[latin1]{inputenc}

\usepackage{bm}
\usepackage{amsmath}
\usepackage{amssymb}  

\usepackage{dsfont}

\usepackage[normalem]{ulem}

\usepackage{dsfont}

\DeclareMathOperator*{\Id}{I}
\DeclareMathOperator*{\VS}{V}

\DeclareMathOperator*{\diag}{diag}

\newcommand{\idof}{\ensuremath{\unlhd}}

\newcommand{\ls}{\ell}

\newcommand{\idS}{\ensuremath{\mathfrak{s}}}
\newcommand{\idN}{\ensuremath{\mathfrak{n}}}
\newcommand{\idM}{\ensuremath{\mathfrak{m}}}

\newcommand{\R}{\ensuremath{\mathds{R}}}		
\newcommand{\C}{\ensuremath{\mathds{C}}}		
\newcommand{\N}{\ensuremath{\mathds{N}}}		

\newcommand{\Gauss}{\ensuremath{\mathcal{N}}}	

\newcommand{\s}{\ensuremath{\mathfrak{s}}}		
\newcommand{\n}{\ensuremath{\mathfrak{n}}}		

\DeclareMathOperator*{\E}{\mathds{E}}					
\DeclareMathOperator*{\argmin}{argmin}		    
\DeclareMathOperator*{\htid}{ht}  		        

\newcommand{\KLD}{D_{\text{KL}}}				
\newcommand{\esi}{\hat{\Sigma}}					
\newcommand{\emu}{{\ensuremath{\hat{\mu}}}}	

\newcommand{\0}{\ensuremath{0}}

\newcommand{\cumu}{\kappa}

\newcommand{\ds}{\displaystyle}

\newcommand{\id}{\operatorname{id}}

\newcommand{\rk}{\operatorname{rank}}

\newcommand{\codim}{\operatorname{codim}}

\newcommand{\lspan}{\operatorname{span}}

\newcommand{\nd}{\Delta (d)}
\newcommand{\nD}{\Delta (D)}

\newcommand{\calA}{\mathcal{A}}

\newcommand{\calI}{\mathcal{I}}

\newcommand{\calM}{\mathcal{M}}

\newcommand{\calO}{\mathcal{O}}
\newcommand{\calP}{\mathcal{P}}

\newcommand{\fraks}{\mathfrak{s}}

\newcommand{\frakp}{\mathfrak{p}}

\newcommand{\frakm}{\mathfrak{m}}

\newcommand{\CC}{\mathbb{C}}

\newcommand{\NN}{\mathbb{N}}

\newcommand{\RR}{\mathbb{R}}

\newtheorem{Satz}{Theorem}[section]
\newtheorem{Thm}[Satz]{Theorem}

\newtheorem{Prop}[Satz]{Proposition}

\newtheorem{Cor}[Satz]{Corollary}

\newtheorem{Lem}[Satz]{Lemma}

\newtheorem{Prob}[Satz]{Problem}

\newtheorem{Rem}[Satz]{Remark}
\newtheorem{Def}[Satz]{Definition}

\newtheorem{Ex}[Satz]{Example}
\newtheorem{Not}[Satz]{Notation}

\newtheorem{Ass}[Satz]{Assumption}

\newcommand{\changeenum}%
{}

\newcommand{\uunder}[2]{{\ds \underline{#1} \atop
{\raise4pt\hbox{$\textstyle \underline{#2}$}}}}

\def\itboxx#1{\ifvmode\indent\fi\makebox[2em][r]{\rmn(#1)} }
\newcommand{\rmn}{\fontshape{n}\fontseries{m}\selectfont\rm}

\jmlrheading{1}{2011}{1-48}{4/17}{5/42}{Franz J.~Kir\'{a}ly, Paul von B\"unau, Frank C.~Meinecke, Duncan A.J.~Blythe and Klaus-Robert Müller}

\ShortHeadings{Algebraic Geometric Comparison of Probability Distributions}{Kir\'{a}ly, von~B\"unau, Meinecke, Blythe and Müller}
\firstpageno{1}

\begin{document}

\title{Algebraic Geometric Comparison of Probability Distributions}

\author{\name Franz J.~Kir\'{a}ly \email franz.j.kiraly@tu-berlin.de \\
       \addr Machine Learning Group, Computer Science \\
        Berlin Institute of Technology (TU Berlin) \\
        Franklinstr.~28/29, 10587 Berlin, Germany\\
        and Discrete Geometry Group, Institute of Mathematics, FU Berlin \\
        \AND
       \name Paul von B\"unau \email paul.buenau@tu-berlin.de
       \AND
       \name Frank C.~Meinecke \email frank.meinecke@tu-berlin.de \\
       \addr Machine Learning Group, Computer Science \\
        Berlin Institute of Technology (TU Berlin) \\
        Franklinstr.~28/29, 10587 Berlin, Germany\\
       \AND
       \name Duncan A.~J.~Blythe \email duncan.blythe@bccn-berlin.de \\
       \addr Machine Learning Group, Computer Science \\
        Berlin Institute of Technology (TU Berlin) \\
        Franklinstr.~28/29, 10587 Berlin, Germany\\
        and Bernstein Center for Computational Neuroscience (BCCN), Berlin \\
       \AND
       \name Klaus-Robert Müller \email klaus-robert.mueller@tu-berlin.de \\
       \addr Machine Learning Group, Computer Science \\
        Berlin Institute of Technology (TU Berlin) \\
        Franklinstr.~28/29, 10587 Berlin, Germany\\
		and IPAM, UCLA, Los Angeles, USA
       }

\editor{The editor}

\maketitle

\begin{abstract}
We propose a novel algebraic algorithmic framework for dealing with probability distributions represented
by their cumulants such as the mean and covariance matrix. As an example, we consider the
unsupervised learning problem of finding the subspace on which several probability
distributions agree. Instead of minimizing an objective function involving the estimated
cumulants, we show that by treating the cumulants as elements of the polynomial ring
we can directly solve the problem, at a lower computational cost and with higher
accuracy. Moreover, the algebraic viewpoint on probability distributions allows us to
invoke the theory of algebraic geometry, which we demonstrate in a compact
proof for an identifiability criterion.
\end{abstract}



\section{Introduction}
\label{sec:Intro}
Comparing high dimensional probability distributions is a general
problem in machine learning, which occurs in two-sample testing
(e.g.~\cite{Hotelling31,Gretton07akernel}), projection pursuit (e.g.~\cite{Fried74}),
dimensionality reduction and feature selection (e.g.~\cite{Tor03Feature}).
Under mild assumptions, probability densities are uniquely determined by their
cumulants which are naturally interpreted as coefficients of
homogeneous multivariate polynomials. Representing probability densities
in terms of cumulants is a standard technique in learning algorithms. For example,
in Fisher Discriminant Analysis~\citep{Fisher36}, the class conditional
distributions are approximated by their first two cumulants.

In this paper, we take this viewpoint further and work explicitly with
polynomials. That is, we treat estimated cumulants not as constants in an objective
function, but as objects that we manipulate algebraically in order to find the
optimal solution. As an example, we consider the problem of finding the linear subspace on
which several probability distributions are identical: given $D$-variate random variables
$X_1, \ldots, X_m$, we want to find the linear map $P \in \R^{d \times D}$ such
that the projected random variables have the same probability distribution,
\begin{align*}
	P X_1 \sim \cdots \sim P X_m .
\end{align*}
This amounts to finding the directions on which all projected cumulants agree. For
the first cumulant, the mean, the projection is readily available as the solution of a
set of linear equations. For higher order cumulants, we need to solve polynomial equations of
higher degree. We present the first algorithm that solves this problem
explicitly for arbitrary degree, and show how algebraic geometry can be applied
to prove properties about it.

\begin{figure}[h]
	\begin{center}
		\includegraphics{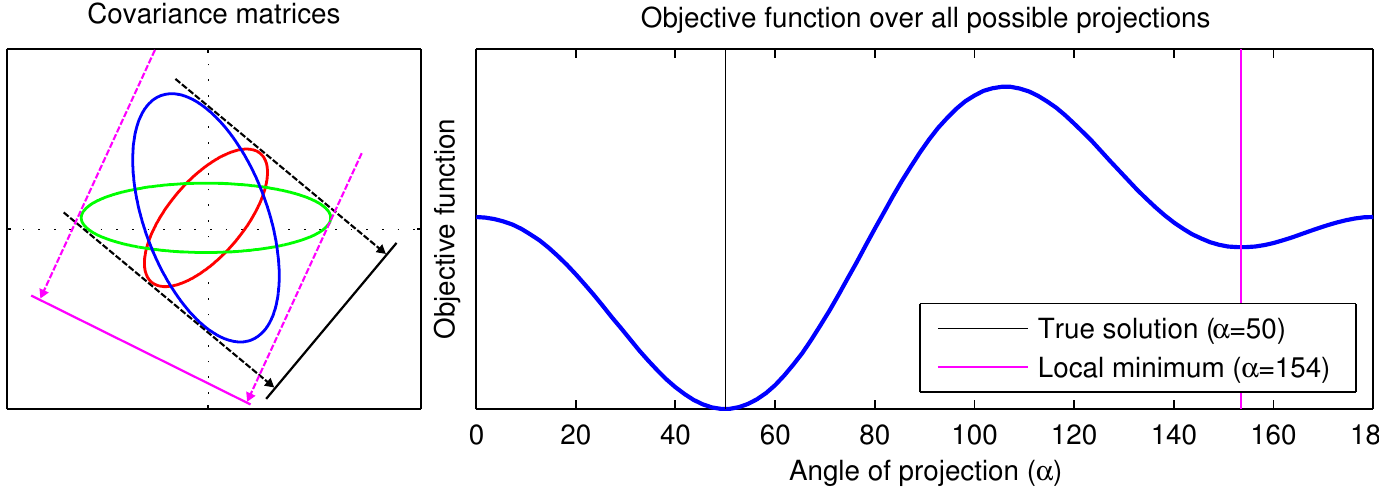}
		\caption{\label{fig:ml_optim}
			Illustration of the optimization approach. The left panel shows the contour plots
			of three sample covariance matrices. The black line is the true one-dimensional subspace
			on which the projected variances are exactly equal, the magenta line corresponds
			to a local minimum of the objective function. The right panel shows the value of
			the objective function over all possible one-dimensional subspaces, parameterized
			by the angle $\alpha$ to the horizontal axis; the angles corresponding to the global minimum
			and the local minimum are indicated by black and magenta lines respectively.
		   }
	\end{center}
\end{figure}

To clarify the gist of our approach, let us consider a stylized example. In order to solve
a learning problem, the conventional approach in machine learning is to formulate an objective function,
e.g.~the log likelihood of the data or the empirical risk.
Instead of minimizing an objective function that involves the polynomials, we consider the polynomials as
\textit{objects in their own right} and then solve the problem by algebraic manipulations.
The advantage of the algebraic approach is that it captures the inherent structure of
the problem, which is in general difficult to model in an optimization approach. In other words, the
algebraic approach actually \textit{solves} the problem, whereas optimization \textit{searches} the space of possible solutions guided by an objective function that is minimal at the desired solution, but can give poor directions outside of the neighborhood around its global minimum. Let us consider the problem where we would like to find the direction $v \in \RR^2$ on which several sample covariance matrices $\Sigma_1, \ldots, \Sigma_m \subset \R^{2 \times 2}$ are equal. The usual ansatz would be to formulate an optimization problem such as
\begin{align}
\label{eq:objfun}
	v^* = \argmin_{\| v \| = 1} \sum_{1 \le i,j \le m} \left( v^\top \Sigma_i v -  v^\top \Sigma_j v \right) ^2 .
\end{align}
This objective function measures the deviation from equality for all pairs of covariance
matrices; it is zero if and only if all projected covariances are equal and positive
otherwise. Figure~\ref{fig:ml_optim} shows an example with three covariance matrices (left panel) and
the value of the objective function for all possible projections
$v = \begin{bmatrix} \cos(\alpha) & \sin(\alpha) \end{bmatrix}^\top$.
The solution to this non-convex optimization problem can be found using a gradient-based search
procedure, which may terminate in one of the local minima (e.g.~the magenta line in Figure~\ref{fig:ml_optim})
depending on the initialization.

However, the natural representation of this problem is not in terms of an objective function, but rather a
system of equations to be solved for $v$, namely
\begin{align}
\label{eq:simple_example}
	v^\top \Sigma_1 v = \cdots = v^\top \Sigma_m v .
\end{align}
In fact, by going from an algebraic description of the set of solutions to a formulation as an
optimization problem in Equation~\ref{eq:objfun}, we lose important structure. In the case where
there is an exact solution, it can be attained explicitly with algebraic
manipulations. However, when we estimate a covariance matrix from finite or
noisy samples, there exists no exact solution in general. Therefore we present an algorithm
which combines the statistical treatment of uncertainty in the coefficients of polynomials with the
exactness of algebraic computations to obtain a consistent estimator for $v$ that is
computationally efficient.

Note that this approach is not limited to this particular learning task. In fact, it is applicable
whenever a set of solutions can be described in terms of a
set of polynomial equations, which is a rather general setting. For example, we could use a similar
strategy to find a subspace on which the projected probability distribution has another property
that can be described in terms of cumulants, e.g.~independence between variables. Moreover, an algebraic
approach may  also be useful in solving certain optimization problems, as the set of extrema of a
polynomial objective function can be described by the vanishing set of its gradient. The algebraic viewpoint also allows a novel interpretation of algorithms operating in the feature space associated with the polynomial kernel. We would therefore argue that methods
from computational algebra and algebraic geometry are useful for the wider machine learning
community.

\begin{figure}[h]
  \begin{center}
    \includegraphics{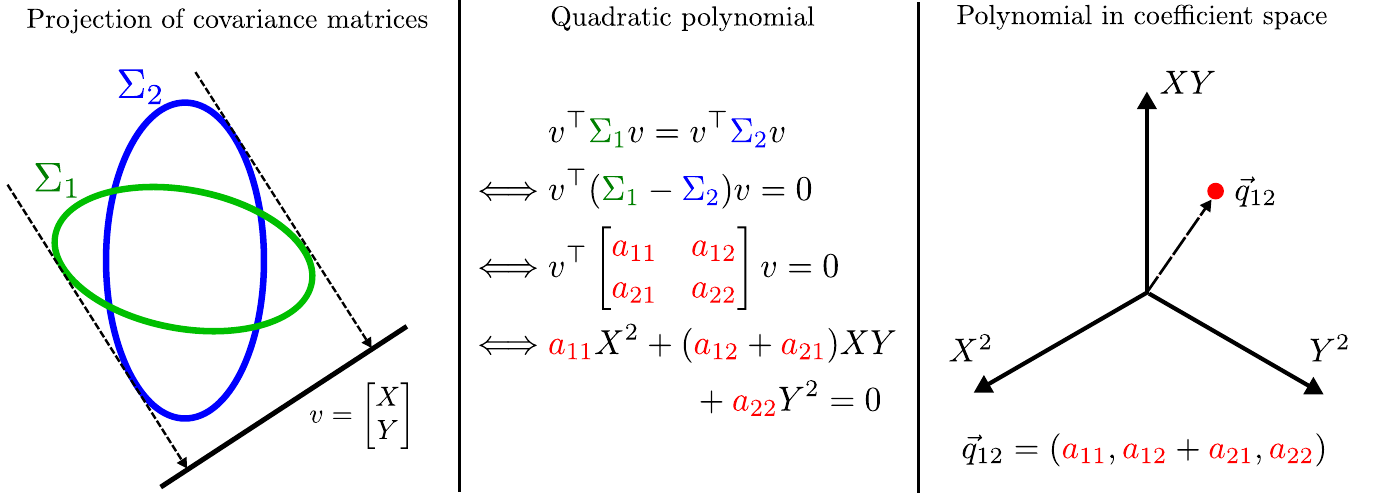}
 	 \caption{
	 	\label{fig:alg_setting}
			Representation of the problem: the left panel shows sample covariance
			matrices $\Sigma_1$ and $\Sigma_2$ with the desired projection $v$.
			In the middle panel, this projection is defined as
			the solution to a quadratic polynomial. This polynomial is
			embedded in the vector space of coefficients spanned by the monomials
			$X^2, Y^2$ and $X Y$ shown in the right panel. 	
			       }
  \end{center}
\end{figure}

Let us first of all explain the representation over which we compute. We will proceed in the three
steps illustrated in Figure~\ref{fig:alg_setting}, from the geometric interpretation of sample
covariance matrices in data space (left panel), to the quadratic equation defining the projection $v$
(middle panel), to the representation of the quadratic equation as a coefficient vector (right panel).
To start with, we consider the Equation~\ref{eq:simple_example} as a set of homogeneous
quadratic equations defined by
\begin{align}
\label{eq:homo_simple_example}
	v^\top ( \Sigma_i - \Sigma_j ) v = 0   \; \; \forall \, 1 \le i,j \le m,
\end{align}
where we interpret the components of $v$ as variables, $v = \begin{bmatrix} X & Y \end{bmatrix}^\top$.
The solution to these equations is the direction in $\R^2$ on which the projected variance is
equal over all covariance matrices. Each of these equations corresponds to a quadratic
polynomial in the variables $X$ and $Y$,
\begin{align}
\label{eq:polyintro}
	q_{ij} & = v^\top (\Sigma_i - \Sigma_j) v \nonumber \\
		   & = v^\top \begin{bmatrix}
		   				a_{11} & a_{12} \\ a_{21} & a_{22}
		   		      \end{bmatrix} v \nonumber \\
		   & = a_{11} X^2 + ( a_{12} + a_{21}) X Y + a_{22} Y^2,
\end{align}
which we embed into the vector space of coefficients. The coordinate axis are the monomials
$\{ X^2, X Y, Y^2 \}$, i.e.~the three independent entries in the Gram matrix $(\Sigma_i - \Sigma_j)$.
That is, the polynomial in Equation~\ref{eq:polyintro} becomes the coefficient vector
\begin{align*}
	\vec{q}_{ij} = \begin{bmatrix} a_{11} & a_{12}+a_{21} & a_{22} \end{bmatrix}^\top .
\end{align*}
The motivation for the vector space interpretation is that every linear combination of the
Equations~\ref{eq:homo_simple_example} is also a characterization of the set of solutions:
this will allow us to find a particular set of equations by linear combination, from
which we can directly obtain the solution. Note, however, that the vector space
representation does not give us all equations which can be used to describe the solution: we
can also multiply with arbitrary polynomials. However, for the algorithm that we present
here, linear combinations of polynomials are sufficient.

\begin{figure}[h]
  \begin{center}
    \includegraphics[width=14cm,keepaspectratio=true]{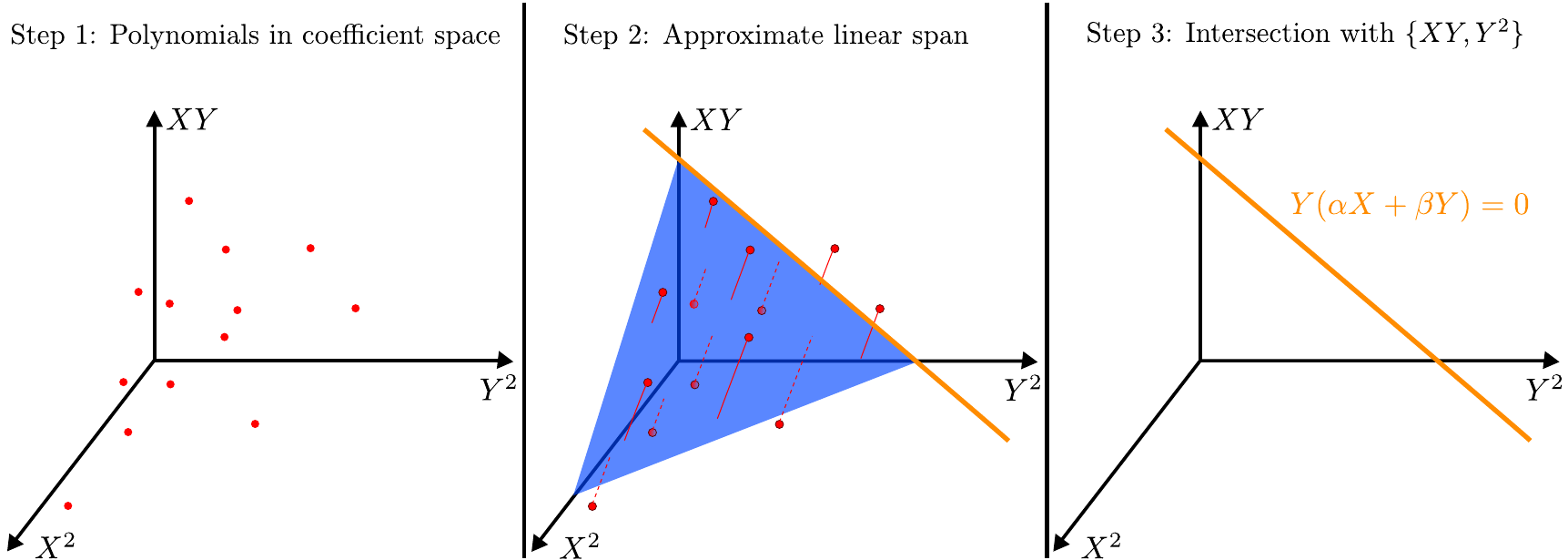}
 	 \caption{
	 	\label{fig:alg_approach}
		Illustration of the algebraic algorithm. The left panel shows the vector space
		of coefficients where the polynomials corresponding to the
		Equations~\ref{eq:homo_simple_example} are considered as elements of the vector space
		shown as red points. The middle
		panel shows the approximate $2$-dimensional subspace (blue surface) onto which we project the
		polynomials. The right panel shows the one-dimensional intersection (orange line) of the approximate
		subspace with the plane spanned by spanned by $\{X Y, Y^2\}$. This subspace is spanned by the polynomial
		$Y ( \alpha X + \beta Y )$, so we can divide by the variable $Y$.
       }
  \end{center}
\end{figure}

Figure~\ref{fig:alg_approach} illustrates how the algebraic algorithm works in the vector space of
coefficients. The polynomials
$\mathcal{Q} = \{ q_{ij} \}_{i,j=1}^n$ span a space of constraints which defines the set of solutions.
The next step is to find a polynomial of a certain form that immediately reveals the solution.
One of these sets is the linear subspace spanned by the monomials $\{X Y, Y^2\}$: any
polynomial in this span is divisible by $Y$.
Our goal is now to find a polynomial which is contained in both this subspace and the span of $\mathcal{Q}$.
Under mild assumptions,
one can always find a polynomial of this form, and it corresponds to an equation
\begin{align}
\label{eq:solution_span}
	Y ( \alpha X + \beta Y ) = 0 .
\end{align}
Since this polynomial is in the span of $\mathcal{Q},$ our solution $v$ has to be a zero of
this particular polynomial: $v_2 ( \alpha v_1 + \beta v_2) = 0$. Moreover, we can
assume\footnote{This is a consequence of the generative model for the observed polynomials which
is introduced in Section~\ref{sec:alg_prob-poly}. In essence, we use the fact that
our polynomials have no special property (apart from the existence of a solution) with probability
one.
}
that $v_2 \neq 0$, so that we can divide out the variable $Y$  to get the linear factor
$( \alpha X + \beta Y )$,
\begin{align*}
	0 = \alpha X + \beta Y = \begin{bmatrix} \alpha & \beta \end{bmatrix} v .
\end{align*}
Hence $v = \begin{bmatrix} -\beta & \alpha  \end{bmatrix}^\top$ is the solution up to arbitrary
scaling, which corresponds to the one-dimensional subspace in Figure~\ref{fig:alg_approach} (orange line, right panel).
A more detailed treatment of this example can also be found in Appendix~\ref{app-example}.

In the case where there exists a direction $v$ on which the projected covariances are
exactly equal, the linear subspace spanned by the set of polynomials $\mathcal{Q}$ has
dimension two, which corresponds to the degrees of freedom of possible covariance matrices
that have fixed projection on one direction. However, since in practice covariance
matrices are estimated from finite and noisy samples, the polynomials $\mathcal{Q}$
usually span the whole space, which means that there exists only a trivial solution $v = 0$.
This is the case for the polynomials pictured in the left panel of Figure~\ref{fig:alg_approach}.
Thus, in order to obtain an approximate solution, we first determine the
approximate two-dimensional span of $\mathcal{Q}$ using a standard least squares method as
illustrated in the middle panel. We can then find the intersection of the approximate
two-dimensional span of $\mathcal{Q}$ with the plane spanned by the monomials $\{X Y, Y^2\}$.
As we have seen in Equation~\ref{eq:solution_span}, the polynomials in this span provide us with
a unique solution for $v$ up to scaling, corresponding to the fact that the intersection has
dimension one (see the right panel of Figure~\ref{fig:alg_approach}). Alternatively, we could
have found the one-dimensional intersection with the span of $\{ X Y, X^2 \}$ and
divided out the variable $X$. In fact, in the final algorithm we will find all
such intersections and combine the solutions in order to increase the accuracy.
Note that we have found this solution by solving a simple least-squares problem (second step, middle
panel of Figure~\ref{fig:alg_approach}). In contrast, the optimization approach
(Figure~\ref{fig:ml_optim}) can require a large number of iterations and may converge
to a local minimum. A more detailed example of the algebraic algorithm can be found
in Appendix~\ref{app-example}.

\begin{figure}[h]
	\begin{center}
    		\includegraphics[width=4cm]{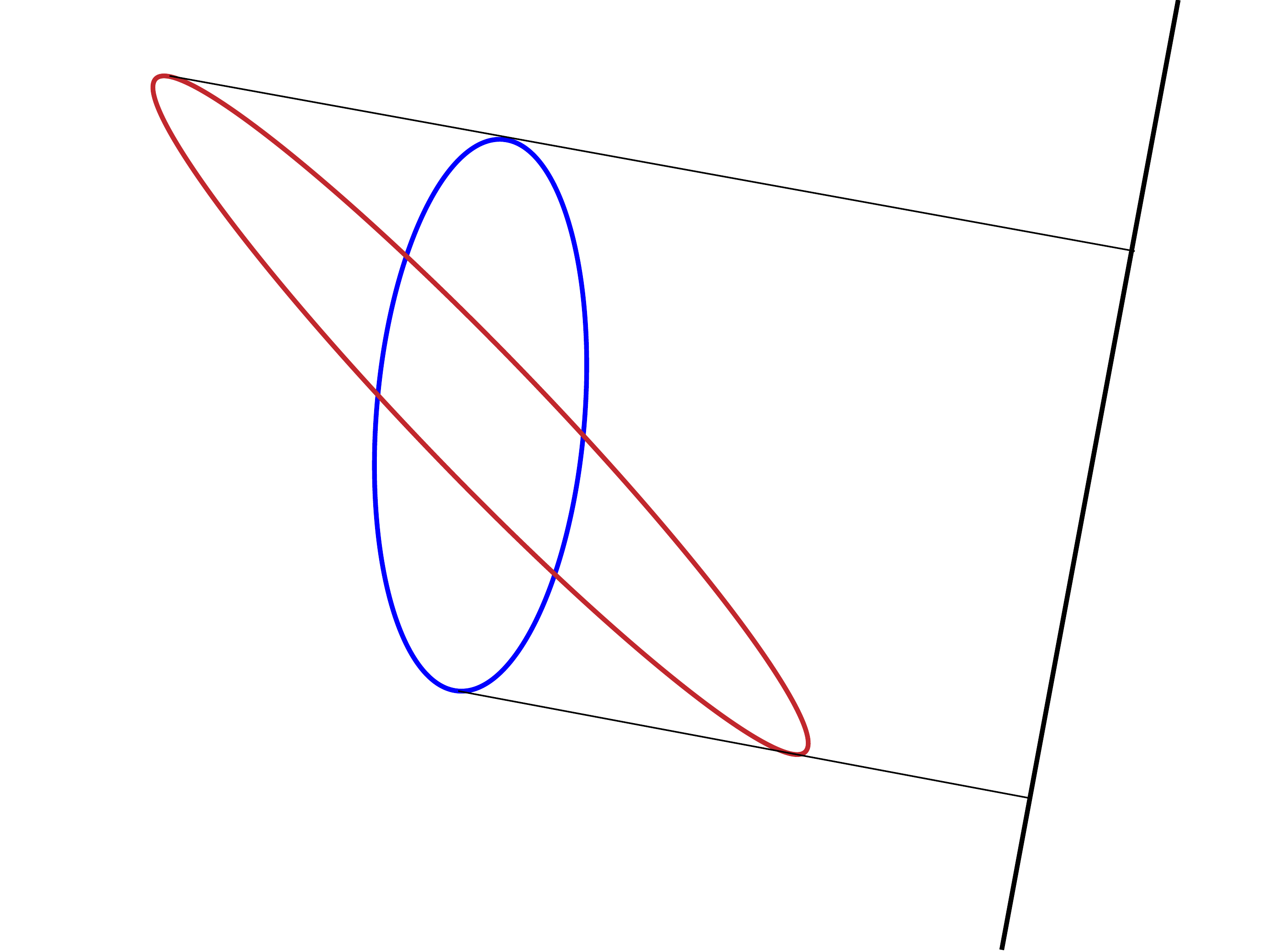}	
    		\includegraphics[width=4cm]{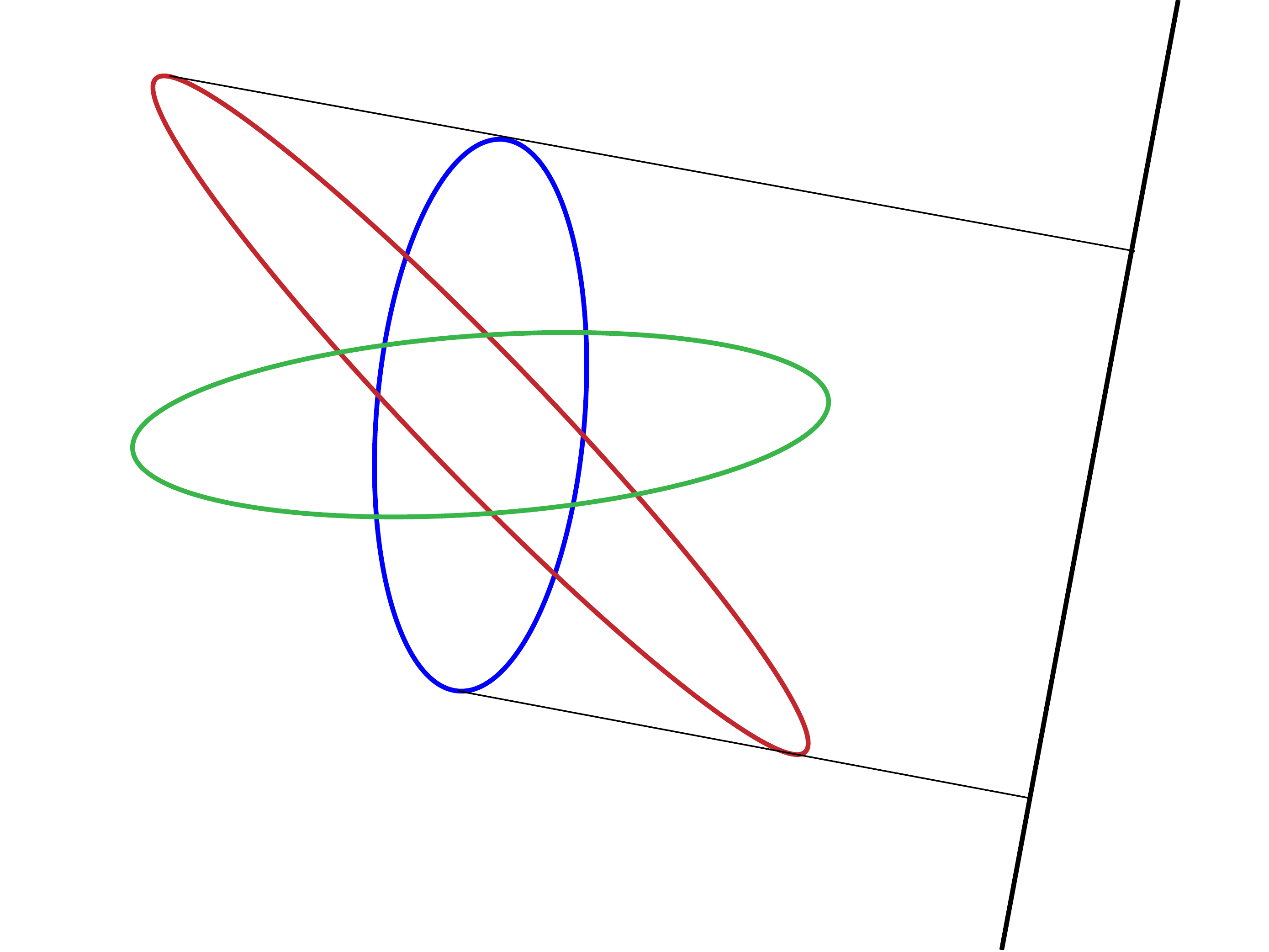}
		
		\caption{\label{fig:spurious_example}
		The left panel shows two sample covariance matrices in the plane, along with
		a direction on which they are equal. In the right panel, a third (green) covariance
		matrix does not have the same projected variance on the black direction.
		}
		
	\end{center}
\end{figure}

The algebraic framework does not only allow us to construct efficient algorithms for
working with probability distributions, it also offers powerful tools to prove properties
of algorithms that operate with cumulants. For example, we can answer the
following central question: how many distinct data sets do we need such that the subspace
with identical probability distributions becomes uniquely identifiable? This depends on
the number of dimensions and the cumulants that we consider.
Figure~\ref{fig:spurious_example} illustrates the case where we are given
only the second order moment in two dimensions. Unless $\Sigma_1 - \Sigma_2$ is
indefinite, there \textit{always} exists a direction
on which two covariance matrices in two dimensions are equal
(left panel of Figure~\ref{fig:spurious_example}) ---
irrespective of whether the probability distributions are actually equal. We therefore
need at least three covariance matrices (see right panel), or to consider other cumulants as well.
We derive a tight criterion on the necessary number of data sets depending on the
dimension and the cumulants under consideration. The proof hinges on viewing the
cumulants as polynomials in the algebraic geometry framework: the polynomials that
define the sought-after projection (e.g.~Equations~\ref{eq:homo_simple_example}) generate
an ideal in the polynomial ring which corresponds to an algebraic set that contains all
possible solutions. We can then show how many independent polynomials are necessary so
that the dimension of the linear part of the algebraic set has smaller dimension in
the generic case. We conjecture that these proof techniques are also applicable to other
scenarios where we aim to identify a property of a probability distribution from
its cumulants using algebraic methods.

Our work is not the first that applies geometric or algebraic methods to Machine
Learning or statistics: for example, methods from group theory have already found their application in machine learning, e.g.~\cite{Ris07Skew,RisBor08Skew}; there are also algebraic methods estimating structured manifold models for data points as in \cite{GPCA05} which are strongly related to polynomial kernel PCA --- a method which can itself be interpreted as a way of finding an approximate vanishing set.

The field of Information Geometry interprets parameter spaces of
probability distributions as differentiable manifolds and studies them from an information-theoretical
point of view (see for example the standard book by \cite{Ama00}), with recent interpretations and
improvements stemming from the field of algebraic geometry by \cite{Wat09}. There is also the nascent field of algebraic statistics which studies the parameter spaces of mainly discrete random variables in terms of commutative algebra and
algebraic geometry, see the recent overviews by \cite[chapter 8]{Stu02} and \cite{Stu10} or the book by
\cite{Gib10} which also focuses on the interplay between information geometry and algebraic statistics.
These approaches have in common that the algebraic and geometric concepts arise naturally
when considering distributions in parameter space.

Given samples from a probability distribution, we may also consider algebraic
structures in the data space. Since the data are uncertain, the algebraic
objects will also come with an inherent uncertainty, unlike the exact
manifolds in the case when we have an a-priori family of probability
distributions. Coping with uncertainties is one of the main interests
of the emerging fields of approximative and numerical commutative
algebra, see the book by \cite{Ste04} for an overview on numerical
methods in algebra, and \citep{KrePouRob09} for recent developments in
approximate techniques on noisy data. There exists a wide range of
methods; however, to our knowledge, the link between approximate algebra
and the representation of probability distributions
in terms of their cumulants has not been studied yet.


The remainder of this paper is organized as follows:
in the next Section~\ref{sec:alg_prob}, we introduce the algebraic view of probability
distribution, rephrase our problem in terms of this framework and investigate its
identifiability. The algorithm for the exact case is presented
in Section~\ref{sec:exact}, followed by the approximate version in
Section~\ref{sec:approx}. The results of our numerical simulations and a
comparison against Stationary Subspace Analysis (SSA)~\citep{PRL:SSA:2009}
can be found in Section~\ref{sec:sims}. In the last Section~\ref{sec:concl},
we discuss our findings and point to future directions. The appendix
contains an example and proof details.

\section{The Algebraic View on Probability Distributions}
\label{sec:alg_prob}

In this section we introduce the algebraic framework for dealing with probability distributions. This requires basic concepts from complex algebraic geometry. A comprehensive introduction to algebraic geometry with a view to computation
can be found in the book~\citep{Cox}. In particular, we recommend to go
through the Chapters~1 and 4.

In this section, we demonstrate the algebraic viewpoint of probability distributions on the
application that we study in this paper: finding the linear subspace on which probability
distributions are equal.
\begin{Prob}\label{Prob:orig}
Let $X_1, \ldots, X_{m}$ be a set of $D$-variate random variables, having smooth densities. Find all linear maps
$P\in \R^{d\times D}$ such that the transformed random variables have the same distribution,
$$P X_1 \sim \cdots \sim P X_m .$$
\end{Prob}
In the first part of this section, we show how this problem can be
formulated algebraically. We will first of all review the relationship between the probability density
function and its cumulants, before we translate the cumulants
into algebraic objects. Then we introduce the theoretical underpinnings for the statistical
treatment of polynomials arising from estimated cumulants and prove conditions on identifiability
for the problem addressed in this paper.

\subsection{From Probability Distributions to Polynomials}
\label{sec:alg_prob-poly}
The probability distribution of every smooth real random variable $X$ can be fully characterized in terms
of its \textit{cumulants}, which are the tensor coefficients of the cumulant generating function.
This representation has the advantage that each cumulant provides a compact description of
certain aspects of the probability density function.
\begin{Def}
Let $X$ be a $D$-variate random variable. Then by $\cumu_n(X)\in \R^{D^{(\times n)}}$ we denote the
$n$-th cumulant, which is a real tensor of degree $n$.
\end{Def}
Let us introduce a useful shorthand notation for linearly transforming tensors.
\begin{Def}
Let $A\in \CC^{d\times D}$ be a matrix. For a tensor $T\in \R^{D^{(\times n)}}$ (i.e.~a real tensor $T$ of degree $n$ of dimension $D^n=D\cdot D\cdot \ldots\cdot D$) we will denote by $A\circ T$ the application of $A$ to $T$ along all tensor dimensions, i.e.
$$\left(A\circ T\right)_{i_1\dots i_n}=\sum_{j_1=1}^D\dots \sum_{j_n=1}^D A_{i_{1}j_{1}}\cdot\ldots\cdot A_{i_{n}j_{n}}T_{j_1\dots j_n}.$$
\end{Def}
The cumulants of a linearly transformed random variable are the multilinearly transformed
cumulants, which is a convenient property when one is looking for a certain linear subspace.
\begin{Prop}
Let $X$ be a real $D$-dimensional random variable and let $A \in \R^{d \times D}$ be a
matrix. Then the cumulants of the transformed random variable $A X$ are the transformed
cumulants,
\begin{align*}
	\cumu_n(A X) = A  \circ \cumu_n(X).
\end{align*}
\end{Prop}
We now want to formulate our problem in terms of cumulants. First of all, note that
$P X_i \sim P X_j$ if and only if $v X_i\sim v X_j$ for all row vectors $v\in \lspan P^\top.$
\begin{Prob}
Find all $d$-dimensional linear subspaces in the set of vectors
\begin{align*}
	S & = \{ v \in \R^D \; \left| \; v^\top X_1 \sim \cdots \sim v^\top X_m \} \right. \\
	  & = \{v\in \R^D \; \left| \; v^\top \circ \cumu_n (X_i)=v^\top \circ \cumu_n (X_j),\; n\in\N, 1 \le i,j\le m\} \right.  .
\end{align*}
\end{Prob}
Note that we are looking for linear subspaces in $S$, but that $S$ itself is not a
vector space in general. Apart from the fact that $S$ is homogeneous, i.e.~$\lambda S = S$ for all
$\lambda \in \R$, there is no additional structure that we utilize.

For the sake of clarity, in
the remainder of this paper we restrict ourselves to the first two cumulants. Note, however, that one
of the strengths of the algebraic framework is that the generalization to arbitrary degree is
straightforward; throughout this paper, we indicate the necessary changes and differences. Thus, from now on, we denote the first two cumulants by
$\mu_i=\cumu_1(X_i)$ and $\Sigma_i=\cumu_2(X_i)$ respectively for all $1 \leq i \leq m$. Moreover,
without loss of generality, we can shift the mean vectors and choose a basis such that the random
variable $X_m$ has zero mean and unit covariance. Thus we arrive at the following formulation.
\begin{Prob}
Find all $d$-dimensional linear subspaces in
$$S=\{v\in \R^D \; | \; v^\top (\Sigma_i-I)v=0,\; v^\top \mu_i=0 , \; 1 \le i \le (m-1)\}.$$
\end{Prob}
Note that $S$ is the set of solutions to $m-1$ quadratic and $m-1$ linear equations in $D$ variables. Now it
is only a formal step to arrive in the framework of algebraic geometry: let us think of the left hand side
of each of the quadratic and linear equations as polynomials $q_1, \ldots, q_{m-1}$ and $f_1, \ldots, f_{m-1}$
in the variables $T_1, \ldots, T_D$ respectively,
\begin{align*}
	q_i = \begin{bmatrix} T_1  \cdots  T_D \end{bmatrix} \circ ( \Sigma_i - I ) 
		\hspace{0.3cm} \text{ and } \hspace{0.3cm}  f_i = \begin{bmatrix} T_1 \cdots T_D \end{bmatrix} \circ \mu_i,
\end{align*}
which are elements of the polynomial ring over the complex numbers in $D$ variables, $\C[T_1,\dots, T_D]$. Note that
in the introduction we have used $X$ and $Y$ to denote the variables in the polynomials, we will now
switch to $T_1, \ldots, T_D$ in order to avoid confusion with random variables.
Thus $S$ can be rewritten in terms of polynomials,
\begin{align*}
	S  = \left\{ v \in \R^D \; \left| \; q_i(v) = f_i(v) = 0 \; \forall \, 1 \leq i \leq m-1 \right\} \right.,
\end{align*}
which means that $S$ is an algebraic set. In the following, we will consider the corresponding complex vanishing set
\begin{align*}
    S &=\VS(q_1, \ldots, q_{m-1}, f_1,\dots, f_{m-1})\\
    &:= \left\{ v \in \C^D \; \left| \; q_i(v) = f_i(v) = 0 \; \forall \, 1 \leq i \leq m-1 \right\} \subseteq \C^D \right.
\end{align*}
and keep in mind that eventually we will be interested in the real part of $S$.
Working over the complex numbers simplifies the theory and creates no algorithmic
difficulties: when we start with real cumulant polynomials, the solution will always
be real. Finally, we can translate our problem into the language of algebraic geometry.
\begin{Prob}\label{Prob:Alg}
Find all $d$-dimensional linear subspaces in the algebraic set
$$S = \VS(q_1, \ldots, q_{m-1}, f_1,\dots, f_{m-1}).$$
\end{Prob}
So far, this problem formulation does not include the assumption that a solution exists. In order to prove
properties about the problem and algorithms for solving it we need to assume that there exist
a $d$-dimensional linear subspace $S' \subset S$. That is, we need to formulate a \textit{generative model} for
our observed polynomials $q_1, \ldots, q_{m-1}, f_1,\dots, f_{m-1}$. To that end, we
introduce the concept of a \textit{generic} polynomial, for a technical definition see Appendix~B. Intuitively, a generic polynomial is a continuous, polynomial valued random variable which almost surely has no algebraic properties except for those that are logically implied by the conditions on it. An algebraic property is an event in the probability space of polynomials which is defined
by the common vanishing of a set of polynomial equations in the coefficients. For example, the property
that a quadratic polynomial is a square of linear polynomial is an algebraic property, since it
is described by the vanishing of the discriminants. In the context of Problem~\ref{Prob:Alg}, we will
consider the observed polynomials as generic conditioned on the algebraic property that they vanish
on a fixed $d$-dimensional linear subspace $S'$.

One way to obtain generic polynomials is to replace coefficients with e.g.~Gaussian random variables.
For example, a generic homogeneous quadric $q \in \C[T_1, T_2]$ is given by
\begin{align*}
	q = Z_{11} T^2_1 + Z_{12} T_1 T_2 + Z_{22} T^2_2,
\end{align*}
where the coefficients $Z_{ij} \sim \Gauss(\mu_{ij}, \sigma_{ij})$ are independent Gaussian random variables with
arbitrary parameters. Apart from being homogeneous, there is no condition on $q$. If we want to add the
condition that $q$ vanishes on the linear space defined by $T_1 = 0$, we would instead consider
\begin{align*}
	q = Z_{11} T^2_1 + Z_{12} T_1 T_2 .
\end{align*}
A more detailed treatment of the concept of genericity, how it is linked to probabilistic sampling, and a comparison with the classical definitions of genericity can be found in Appendix~\ref{sec:gendef}.

We are now ready to reformulate the genericity conditions on the random variables $X_1,\dots, X_m$ in the above framework. Namely, we have assumed that the $X_i$ are general under the condition that they agree in the first two cumulants when projected onto some linear subspace $S'$. Rephrased for the cumulants, Problems~\ref{Prob:orig} and \ref{Prob:Alg} become well-posed and can be formulated as follows.
\begin{Prob}\label{Prob:Alg_gen}
Let $S'$ be an unknown $d$-dimensional linear subspace in $\C^D$. Assume that $f_1,\dots, f_{m-1}$ are generic homogenous linear polynomials, and $q_1,\dots, q_{m-1}$ are generic homogenous quadratic polynomials, all vanishing on $S'.$
Find all $d$-dimensional linear subspaces in the algebraic set
$$S=\VS(q_1, \ldots, q_{m-1}, f_1,\dots, f_{m-1}).$$
\end{Prob}
As we have defined ``generic'' as an implicit ``almost sure'' statement, we are in fact looking for an algorithm which gives the correct answer with probability one under our model assumptions. Intuitively, $S'$ should be also the only
$d$-dimensional linear subspace in $S$, which is not immediately guaranteed from the problem description. Indeed this
is true if $m$ is large enough, which is the topic of the next section.

\subsection{Identifiability}
\label{sec:alg_prob-ident}
In the last subsection, we have seen how to reformulate our initial Problem~\ref{Prob:orig} about comparison of cumulants as the completely algebraic Problem~\ref{Prob:Alg_gen}. We can also reformulate identifiability of the true solution in the original problem in an algebraic way: identifiability in Problem~\ref{Prob:orig} means that the projection $P$ can be uniquely computed from the probability distributions. Following the same reasoning we used to arrive at the algebraic formulation in Problem~\ref{Prob:Alg_gen}, one concludes that identifiability is equivalent to the fact that there exists a unique linear subspace in $S$.

Since identifiability is now a completely algebraic statement, it can be treated also in algebraic terms. In Appendix~\ref{app-generic}, we give an algebraic geometric criterion for identifiability of the stationary subspace; we will sketch its derivation in the following.

The main ingredient is the fact that, intuitively spoken, every generic polynomials carries one degree of freedom in terms of dimension, as for example the following result on generic vector spaces shows:

\begin{Prop}\label{Prop:GenVec-main}
Let $\calP$ be an algebraic property such that the polynomials with property $\mathcal{P}$ form a vector space $V$. Let $f_1,\dots, f_n\in \C[T_1,\dots T_D]$ be generic polynomials satisfying $\mathcal{P}$. Then
$$\rk \lspan (f_1,\dots, f_n)=\min (n, \dim V).$$
\end{Prop}
\begin{proof}
This is Proposition \ref{Prop:GenVec} in the appendix.
\end{proof}

On the other hand, if the polynomials act as constraints, one can prove that each one reduces the degrees of freedom in the solution by one:

\begin{Prop}\label{Thm:genintcont}
Let $Z$ be a sub-vector space of $\C^D$. Let $f_1,\dots, f_n$ be generic homogenous polynomials in $D$ variables (of fixed, but arbitrary degree each), vanishing on $Z$. Then for their common vanishing set $\VS (f_1,\dots, f_n)=\{x\in\C^D\mid f_i(x)=0\;\forall i\}$, one can write
$$\VS (f_1,\dots, f_n) =  Z \cup U,$$
where $U$ is an algebraic set with
$$\dim U\le \max (D-n,\; 0).$$
\end{Prop}
\begin{proof}
This follows from Corollary \ref{Cor:KrullHt-Hom} in the appendix.
\end{proof}
Proposition \ref{Thm:genintcont} can now be directly applied to Problem~\ref{Prob:Alg_gen}. It implies that $S=S'$ if $2(m-1)\ge D+1$, and that $S'$ is the maximal dimensional component of $S$ if $2(m-1) \ge D-d+1$. That is, if we start with $m$ random variables, then $S'$ can be identified uniquely if
$$2(m-1) \ge D-d+1$$
with classical algorithms from computational algebraic geometry in the noiseless case.
\begin{Thm}\label{Thm:ident}
Let $X_1,\dots, X_m$ be random variables. Assume there exists a projection $P\in \R^{d\times D}$ such that the first two cumulants of all $P X_1, \ldots, P X_m$ agree and the cumulants are generic under those conditions. Then the projection $P$ is identifiable from the first two cumulants alone if
$$m \ge \frac{D-d+1}{2}+1.$$
\end{Thm}
\begin{proof}
This is a direct consequence of Proposition \ref{Prop:ident_ex} in the appendix, applied to the reformulation given in Problem \ref{Prob:Alg_gen}. It is obtained by applying Proposition \ref{Thm:genintcont} to the generic forms vanishing on the fixed linear subspace $S'$, and using that $S'$ can be identified in $S$ if it is the biggest dimensional part.
\end{proof}
We have seen that identifiability means that there is an algorithm to compute $P$ uniquely when the cumulants are known, resp.~to compute a unique $S$ from the polynomials $f_i,q_i$. It is not difficult to see that an algorithm doing this can be made into a consistent estimator when the cumulants are sample estimates. We will give an algorithm of this type in the following parts of the paper.

\section{An Algorithm for the Exact Case}
\label{sec:exact}
In this section we present an algorithm for solving Problem~\ref{Prob:Alg_gen}, under the assumption that the
cumulants are known exactly. We will first fix notation and introduce important algebraic concepts.
In the previous section, we derived in Problem~\ref{Prob:Alg_gen} an algebraic formulation of our
task: given generic quadratic polynomials $q_1,\dots, q_{m-1}$ and linear polynomials $f_1,\dots, f_{m-1}$, vanishing on a unknown linear subspace $S'$ of $\C^D$, find $S'$ as the unique $d$-dimensional linear subspace in the algebraic set $\VS(q_1, \ldots, q_{m-1}, f_1,\dots, f_{m-1})$. First of all, note that the linear equations $f_i$ can easily be removed from the problem: instead of looking
at $\C^D$, we can consider the linear subspace defined by the $f_i$, and examine the algebraic set $\VS(q'_1, \ldots, q'_{m-1}),$ where $q'_i$ are polynomials in $D-m+1$ variables which we obtain by substituting ${m-1}$ variables. So the problem we need to examine is in fact the modified problem where we have only quadratic polynomials.
Secondly, we will assume that ${m-1}\ge D$. Then, from Proposition~\ref{Thm:genintcont}, we know that $S=S'$ and
Problem~\ref{Prob:Alg_gen} becomes the following.
\begin{Prob}\label{Prob:Alg_SSA_naive}
Let $S$ be an unknown $d$-dimensional subspace of $\C^D$. Given ${m-1} \geq D$ generic homogenous quadratic polynomials $q_1,\dots, q_{m-1}$ vanishing on $S$, find the $d$-dimensional linear subspace
$$S=\VS(q_1, \ldots, q_{m-1}).$$
\end{Prob}
Of course, we have to say what we mean by \textit{finding} the solution. By assumption, the quadratic
polynomials already fully describe the linear space $S$. However, since $S$ is a linear space, we want
a basis for $S$, consisting of $d$ linearly independent vectors in $\C^D$. Or, equivalently, we want to find linearly independent linear forms $\ls_1,\dots, \ls_{D-d}$ such that $\ell_i(x)=0$ for all $x\in S.$ The latter is the correct description of the solution in algebraic terms. We now show how to reformulate this in the right language, following the algebra-geometry duality. The algebraic set $S$ corresponds to an ideal in the polynomial ring $C[T_1,\dots, T_D]$.
\begin{Not}
We denote the polynomial ring $\C[T_1,\dots, T_D]$ by $R$. The ideal of $S$ is an ideal in $R$, and we denote it by by $\idS=\Id (S).$ Since $S$ is a linear space, there exists a linear generating set $\ls_1, \ldots, \ls_{D-d}$ of $\idS$ which we will fix in the following.
\end{Not}
We can now relate the Problem~\ref{Prob:Alg_SSA_naive} to a classical problem in algebraic geometry.
\begin{Prob}\label{Prob:Alg_SSA}
Let $m > D$ and  $q_1,\dots, q_{m-1}$ be generic homogenous quadratic polynomials vanishing on a linear $d$-dimensional subspace $S\subseteq \C^D$.
Then find a linear basis for the radical ideal
\begin{align*}
	\sqrt{\langle q_1, \ldots, q_{m-1}\rangle} = \Id(\VS(q_1, \ldots, q_{m-1}))=\Id (S).
\end{align*}
\end{Prob}
The first equality follows from Hilbert's Nullstellensatz. This also shows that solving the problem is in fact a question of computing a radical of an ideal. Computing the radical of an ideal is a classical problem in computational algebraic geometry, which is known to be difficult
(for a more detailed discussion see Section~\ref{sec:exact-prwork}). However, if we assume ${m-1}\ge D(D+1)/2 - d(d+1)/2$, we can dramatically reduce the computational cost and it is straightforward to derive an approximate solution. In this case, the $q_i$ generate the vector space of homogenous quadratic polynomials which vanish on $S$, which we will denote by $\idS_2.$ That this is indeed the case, follows from Proposition \ref{Prop:GenVec-main}, and we have
$\dim \idS_2 = D(D+1)/2 - d(d+1)/2,$
as we will calculate in Remark \ref{Rem:gensm}.

Before we continue with solving the problem, we will need to introduce several concepts and abbreviating notations. First we introduce notation to denote sub-vector spaces which contain polynomials of certain degrees.
\begin{Not}
Let $\mathcal{I}$ be a sub-$\mathbb{C}$-vector space of $R$, i.e.~$\mathcal{I}=R$, or $\mathcal{I}$ is some ideal of $R$, e.g.~$\mathcal{I}=\idS.$ We denote the sub-$\mathbb{C}$-vector space of homogenous polynomials of degree $k$ in $\mathcal{I}$ by $\mathcal{I}_k$ (in commutative algebra, this is standard notation for homogenously generated $R$-modules).
\end{Not}
For example, the homogenous polynomials of degree $2$ vanishing on $S$ form exactly the vector space $\idS_2$. Moreover, for any $\mathcal{I}$, the equation $\mathcal{I}_k=\mathcal{I}\cap R_k$ holds. The vector spaces $R_2$ and $\idS_2$ will be the central objects in the following chapters. As we have seen, their dimension is given in terms of triangular numbers, for which we introduce some notation:

\begin{Not}
We will denote the $n$-th triangular number by $\Delta (n) = \frac{n(n+1)}{2}$.
\end{Not}

The last notational ingredient will capture the structure which is imposed on $R_k$ by the orthogonal decomposition $\C^D = S \oplus S^\perp.$
\begin{Not}
Let $S^\perp$ be the orthogonal complement of $S$. Denote its ideal by $\idN = \Id\left(S^\perp\right)$.
\end{Not}
\begin{Rem}\label{Rem:calcidS}
As $\idN$ and $\idS$ are homogenously generated in degree one, we have the calculation rules
\begin{align*}
\idS_{k+1}=\idS_k\cdot R_1\quad\mbox{and}\quad \idN_{k+1}=\idN_k\cdot R_1,\\
(\idS_1)^k=(\idS^k)_k\quad\mbox{and}\quad (\idN_1)^k=(\idN^k)_k
\end{align*}
where $\cdot$ is the symmetrized tensor or outer product of vector spaces (these rules are canonically induced by the so-called graded structure of $R$-modules). In terms of ideals, the above decomposition translates to
$$R_1=\idS_1\oplus \idN_1.$$
Using the above rules and the binomial formula for ideals, this induces an orthogonal decomposition
\begin{align*}
R_2=&R_1\cdot R_1=(\idS_1\oplus \idN_1)\cdot (\idS_1\oplus \idN_1)= (\idS_1)^2\oplus (\idS_1\cdot\idN_1) \oplus (\idN_1)^2\\
& = \idS_1\cdot(\idS_1\oplus\idN_1) \oplus (\idN^2)_2 = \idS_1\cdot R_1 \oplus (\idN^2)_2 =  \idS_2 \oplus (\idN^2)_2
\end{align*}
(and similar decompositions for the higher degree polynomials $R_k$).
\end{Rem}

The tensor products above can be directly translated to products of ideals, as the vector spaces above are each generated in a single degree (e.g.~$\idS^k, \idN^k$, are generated homogenously in degree $k$). To express this, we will define an ideal which corresponds to $R_1$:
\begin{Not}
We denote the ideal of $R$ generated by
all monomials of degree $1$ by
$\idM = \langle T_1,\dots, T_D \rangle $.
\end{Not}
Note that ideal $\idM$ is generated by all elements in $R_1.$ Moreover, we have $\idM_k=R_k$ for all $k\ge 1$. Using $\idM$, one can directly translate products of vector spaces involving some $R_k$ into products of ideals:
\begin{Rem}\label{Rem:calcidM}
The equality of vector spaces
$$\idS_{k}=\idS_1\cdot (R_1)^{k-1}$$
translates to the equality of ideals
$$\idS\cap \idM^k= \idS\cdot \idM^{k-1},$$
since both the left and right sides are homogenously generated in degree $k$.
\end{Rem}

\subsection{The Algorithm}
\begin{table}[h]
\begin{center}
\begin{tabular}{l|l}
  $S \subset \C^D$ &  $d$-dimensional projection space \\
  $R=\C [T_1,\dots T_D]$ & Polynomial ring over $\C$ in $D$ variables \\
  $R_k$ & $\C$-vector space of homogenous $k$-forms in $T_1, \ldots, T_D$\\
  $\Delta (n) =\frac{n(n+1)}{2}$ & $n$-th triangular number \\
  $\idS=\langle\ls_1,\dots, \ls_{D-d}\rangle = \Id(S)  $ & The ideal of $S$, generated by linear polynomials $\ls_i$\\
  $\idS_k = R_k\cap \idS $   & $\C$-vector space of homogenous $k$-forms vanishing on $S$\\
  $\idN = \Id(S^\perp)$ & The ideal of $S^\perp$\\
  $\idN_k = R_k\cap \idN $   & $\C$-vector space of homogenous $k$-forms vanishing on $S^\perp$\\
  $\idM=\langle T_1,\dots, T_D\rangle$ & The ideal of the origin in $\C^D$
\end{tabular}
\caption{
	Notation and important definitions
	\label{table:notation}
}
\end{center}
\end{table}
In this section we present an algorithm for solving Problem~\ref{Prob:Alg_SSA}, the
computation of the radical of the ideal ${\langle q_1,\dots, q_{m-1}\rangle}$ under the assumption that
$$m\ge \Delta(D)-\Delta(d)+1.$$
Under those conditions, as we will prove in Remark \ref{Rem:gensm} (iii), we have that
$$	\langle q_1,\dots, q_{m-1}\rangle =\idS_2.$$
Using the notations previously defined, one can therefore infer that solving Problem~\ref{Prob:Alg_SSA} is equivalent to computing the radical
$\idS=\sqrt{\idS\cdot \idM}$ in the sense of obtaining a linear generating set for $\idS$,
or equivalent to finding a basis for $\idS_1$ when $\idS_2$
is given in an arbitrary basis. $\idS_2$ contains the complete information
given by the covariance matrices and $\idS_1$ gives an explicit linear
description of the space of projections under which the random variables
$X_1, \ldots, X_m$ agree.

\begin{algorithm}[h]
\caption{\label{alg:exact_covonly} The \textit{input} consists of the
quadratic forms $q_1, \ldots, q_{m-1} \in R$, generating $\idS_2,$ and the dimension $d$;
the \textit{output} is the linear generating set
$\ls_1, \ldots, \ls_{D-d}$ for $\idS_1$.
 }

\begin{algorithmic}[1]
	\State Let $\pi \gets (1 \, 2 \, \cdots \, D)$ be a transitive permutation of the variable indices $\{ 1, \ldots, D \}$

	\State Let $Q \gets \begin{bmatrix} q_1 & \cdots & q_{m-1} \end{bmatrix}^\top$ be the $((m-1) \times \nD)$-matrix
		of coefficient vectors, where every row corresponds to a polynomial and every column to a monomial $T_i T_j$.
		
	\For	{$k=1, \ldots, D-d$} \label{al:for}
	
		\State
			\begin{minipage}[t]{14cm}
				Order the columns of $Q$ according to the lexicographical ordering of monomials $T_i T_j$
				with variable indices permuted by $\pi^k$, i.e.~the ordering of the columns is given
				by the relation $\succ$ as
				\begin{align*}
					T_{\pi^k(1)}^2 & \succ T_{\pi^k(1)}T_{\pi^k(2)} \succ
					T_{\pi^k(1)}T_{\pi^k(3)}\succ \dots \succ T_{\pi^k(1)}T_{\pi^k(D)} \succ
					T_{\pi^k(2)}^2\\
					&\succ T_{\pi^k(2)}T_{\pi^k(3)}\succ \dots \succ T_{\pi^k(D-1)}^2 \succ
					T_{\pi^k(D-1)}T_{\pi^k(D)}\succ T_{\pi^k(D)}^2
				\end{align*}
			\end{minipage}
			
		\State Transform $Q$ into upper triangular form $Q'$ using Gaussian elimination
		
		\State
			\begin{minipage}[t]{14cm}
				The last non-zero row of $Q'$ is a polynomial $T_{\pi^k(D)} \ell$, where $\ell$
				is a linear form in $\idS$, and we set $\ls_k \gets \ell$\label{alg:exact-crux}
			\end{minipage}
		
		\EndFor
\end{algorithmic}
\end{algorithm}

Algorithm~\ref{alg:exact_covonly} shows the procedure in pseudo-code; a summary of the notation defined
in the previous section can be found in Table~\ref{table:notation}. The algorithm has polynomial complexity
in the dimension $d$ of the linear subspace $S$.
\begin{Rem}\label{Rem:compAlgI}
Algorithm~\ref{alg:exact_covonly} has average and worst case complexity
\begin{align*}
	O\left( (\nD-\nd)^2\nD \right),
\end{align*}
In particular, if $d$ is not considered as parameter of the algorithm, the average and the worst case complexity is
$O(D^6).$ On the other hand, if $\nD-\nd$ is considered a fixed parameter, then Algorithm 1 has average and worst case complexity $O(D^2).$
\end{Rem}
\begin{proof}
This follows from the complexities of the elementary operations:
upper triangularization of a generic matrix of rank $r$ with $m$
columns matrix needs $O(r^2m)$ operations. We first perform triangularization of a rank
$\nD-\nd$ matrix with $\nD$ columns. The permutations can be obtained efficiently by
bringing $Q$ in row-echelon form and then performing row operations. Operations for extracting
the linear forms and comparisons with respect to the monomial ordering are negligible. Thus the overall
operation complexity to calculate $\fraks_1$ is $O((\nD-\nd)^2 \nD).$

Note that the difference between worst- and average case lies at most
in the coefficients, since the inputs are generic and the complexity
only depends on the parameter $D$ and not on the $q_i$. Thus, with
probability $1,$ exactly the worst-case-complexity is attained.
\end{proof}

There are two crucial facts which need to be verified for correctness of this algorithm. Namely, there are implicit claims made in Line~\ref{alg:exact-crux} of Algorithm~\ref{alg:exact_covonly}: First, it is claimed that the last non-zero row of $Q'$ corresponds to a polynomial which factors into certain linear forms. Second, it is claimed that the $\ell$ obtained in step~\ref{alg:exact-crux} generate $\fraks$ resp.~$\fraks_1$. The proofs of these non-trivial claims can be found in Proposition~\ref{Prop:Alg1corr} in the next subsection.

Dealing with additional linear forms $f_1,\dots, f_{m-1},$ is possible by way of a slight modification of the algorithm. Because the $f_i$ are linear forms, they are generators of $\fraks.$ We may assume that the $f_i$ are linearly independent. By performing Gaussian elimination before the execution of Algorithm~\ref{alg:exact_covonly}, we may reduce the number of variables by ${m-1}$, thus having to deal with new quadratic forms in $D-m+1$ instead of $D$ variables. Also, the dimension of the space of projections is reduced to $\min(d-m+1, -1).$ Setting $D'=D-m+1$ and $d'=\min (d-m+1,-1)$ and considering the quadratic forms $q_i$ with Gaussian eliminated variables, Algorithm~\ref{alg:exact_covonly} can be applied to the quadratic forms to find the remaining generators for $\idS_1.$ In particular, if $m-1\ge d,$ then there is no need for considering the quadratic forms, since $d$ linearly independent linear forms already suffice to determine the solution.

We can also incorporate forms of higher degree corresponding to higher order cumulants. For this, we start with $\idS_k,$ where $k$ is the degree of the homogenous polynomials we get from the cumulant tensors of higher degree. Supposing we start with enough cumulants, we may assume that we have a basis of $\idS_k.$ Performing Gaussian elimination on this basis with respect to the lexicographical order, we obtain in the last row a form of type $T_{\pi^k(D)}^{k-1}\ell,$ where $\ell$ is a linear form. Doing this for $D-d$ permutations again yields a basis for $\idS_1.$

Moreover, slight algebraic modifications of this strategy also allow to consider data from cumulants of different degree simultaneously, and to reduce the number of needed polynomials to $O(D)$; however, due to its technicality, this is beyond the scope of the paper. We sketch the idea: In the general case, one starts with an ideal
$$\calI=\langle f_1,\dots, f_m\rangle,$$
homogenously generated in arbitrary degrees.
such that $\sqrt{\calI}=\fraks.$ Proposition~\label{Prop:dehom-rad-generic} in the appendix implies that this happens whenever $m\ge D+1.$ One then proves that due to the genericity of the $f_i,$ there exists an $N$ such that
$$\calI_N=\fraks_N,$$
which means that $\fraks_1$ can again be obtained by calculating the saturation of the ideal $\calI$.
When fixing the degrees of the $f_i$, we will have $N=O(D)$ with a relatively small constant (for all $f_i$ quadratic, this even becomes $N=O(\sqrt{D})$). So algorithmically, one would first calculate $\calI_N=\fraks_N,$ which then may be used to compute $\fraks_1$ and thus $\fraks$ analogously to the case $N=2$, as described above.

\subsection{Proof of correctness}
\label{sec:exact-correct}

In order to prove the correctness of Algorithm~\ref{alg:exact_covonly}, we need to prove
the following three statements.
\begin{Prop}\label{Prop:Alg1corr}
For Algorithm~\ref{alg:exact_covonly} it holds that\\
\itboxx{i} $Q$ is of rank $\nD-\nd$.\\
\itboxx{ii} The last column of $Q$ in step 6 is of the claimed form.\\
\itboxx{iii} The $\ls_1, \ldots, \ls_{D-d}$ generate $\idS_1$.
\end{Prop}
\begin{proof}
This proposition will be proved successively in the following:
(i) will follow from Remark~\ref{Rem:gensm} (iii);
(ii) will be proved in Lemma~\ref{Lem:Xell}; and
(iii) will be proved in Proposition~\ref{Prop:gens}.
\end{proof}
Let us first of all make some observations about the structure of the vector space
$\idS_2$ in which we compute. It is the vector space of polynomials of homogenous degree $2$ vanishing on $S$.
On the other hand, we are looking for a basis $\ls_1,\dots, \ls_{D-d}$ of $\idS_1$. The following remark will relate both vector spaces:
\begin{Rem}\label{Rem:gensm}
The following statements hold:\\
\itboxx{i} $\idS_2$ is generated by the polynomials $\ls_iT_j, 1\le i\le D-d, 1\le j\le D, .$ \\
\itboxx{ii} $\dim_\C \idS_2=\nD-\nd$\\
\itboxx{iii} Let $q_1, \ldots, q_m$ with $m\ge \nD-\nd$ be generic homogenous quadratic polynomials in $\fraks$. Then
$\langle q_1,\dots,q_m \rangle=\idS_2.$
\end{Rem}
\begin{proof}
(i)~In Remark~\ref{Rem:calcidS}, we have concluded that $\idS_2=\idS_1\cdot R_1.$ Thus the product vector space $\idS_2$ is generated by a product basis of $\idS_1$ and $R_1$. Since $T_j,1\le j\le D$ is a basis for $R_1$, and $\ls_i,1\le i\le D-d$ is a basis for $\idS_1$, the statement holds.
(ii)~In Remark~\ref{Rem:calcidM}, we have seen that $R_2=\idS_2\oplus (\idN_1)^2,$ thus $\dim \idS_2= \dim R_2 - \dim (\idN_1)^2.$ The vector space $R_2$ is minimally generated by the monomials of degree $2$ in $T_1,\dots T_D$, whose number is $\nD$. Similarly, $(\idN_1)^2$ is minimally generated by the monomials of degree $2$ in the variables $\ls'_1,\dots, \ls'_d$ that form the dual basis to the $\ls_i$. Their number is $\nd$, so the statement follows.
(iii)~As the $q_i$ are homogenous of degree two and vanish on $S$, they are elements in $\idS_2.$  Due to (ii), we can apply Proposition~\ref{Prop:GenVec-main} to conclude that they generate $\idS_2$ as vector space.
\end{proof}

Now we continue to prove the remaining claims.
\begin{Lem}\label{Lem:Xell}
In Algorithm~\ref{alg:exact_covonly} the $(\nD-\nd)$-th row of $Q'$ (the upper triangular form of $Q$) corresponds
to a $2$-form $T_{\pi(D)}\ell$ with a linear polynomial $\ell\in \idS_1$.
\end{Lem}
\begin{proof}
Note that every homogenous polynomial of degree $k$ is canonically an element of the vector space $R_k$ in the monomial basis given by the $T_i$. Thus it makes sense to speak about the coefficients of $T_i$ for an $1$-form resp.~the coefficients of $T_iT_j$ of a $2$-form.

Also, without loss of generality, we can take the trivial permutation $\pi=\id$, since the proof
 will not depend on the chosen lexicographical ordering and
thus will be naturally invariant under permutations of variables.
First we remark: since $S$ is a generic $d$-dimensional
linear subspace of $\C^D$, any linear form in $\idS_1$ will have at least $d+1$ non-vanishing coefficients in the $T_i.$ On the other hand, by displaying the generators $\ls_i, 1\le i\le D-d$ in $\idS_1$ in
reduced row echelon form with respect to the $T_i$-basis, one sees that one can choose all the
$\ls_i$ in fact with exactly $d+1$ non-vanishing coefficients in the $T_i$ such that no nontrivial linear combination of the $\ls_i$ has less then $d+1$ non-vanishing coefficients. In particular, one can choose the $\ls_i$ such that the biggest (w.r.t.~the lexicographical order) monomial with non-vanishing coefficient of $\ls_i$ is $T_i$.

Remark~\ref{Rem:gensm} (i) states that $\idS_2$ is generated by
$$\ls_iT_j, 1\le i\le D-d, 1\le j\le D.$$
Together with our above reasoning, this implies the following.

{\bf Fact 1:} There exist linear forms $\ls_i,1\le i\le D-d$ such that: the $2$-forms $\ls_iT_j$ generate $\idS_2,$ and the biggest monomial of $\ls_iT_j$ with non-vanishing coefficient under the lexicographical ordering is $T_iT_j.$
By Remark~\ref{Rem:gensm} (ii), the last row of the
upper triangular form $Q'$ is a polynomial which has zero coefficients
for all monomials possibly except the $\nd+1$ smallest,
$$T_{D-d}T_D,T_{D-d+1}^2, T_{D-d+1}T_{D-d+2},\dots, T_{D-1}T_D,T_D^2.$$
On the other hand, it is guaranteed by our genericity assumption that the biggest of those terms is indeed non-vanishing, which implies the following.

{\bf Fact 2:} The biggest monomial of the last row with non-vanishing coefficient (w.r.t~the lexicographical order) is that of $T_{D-d}T_D.$

Combining Facts 1 and 2, we can now infer that the last row must be a scalar multiple of $\ls_{D-d}T_D$: since the last row corresponds to an element of $\idS_2,$ it must be a linear combination of the $\ls_iT_j.$ By Fact 1, every contribution of an $\ls_iT_j, (i,j)\neq (D-d,D)$ would add a non-vanishing coefficient lexicographically
bigger than $T_{D-d}T_D$ which cannot cancel. So, by Fact 2, $T_D$ divides the last row of
the upper triangular form of $Q,$ which then must be $T_D \ls_{D-d}$ or a multiple thereof. Also we have that $\ls_{D-d}\in\fraks$ by definition.
\end{proof}
It remains to be shown that by permutation of the variables we can find a basis for $\idS_1$.
\begin{Prop}\label{Prop:gens}
The $\ell_1,\dots,\ell_{D-d}$ generate $\idS_1$ as vector space and thus $\idS$ as ideal.
\end{Prop}
\begin{proof}
Recall that $\pi^i$ was the permutation to obtain $\ell_i.$ As we have seen in the
proof of Lemma~\ref{Lem:Xell}, $\ell_i$ is a linear form which has
non-zero coefficients only for the $d+1$ coefficients
$T_{\pi^i(D-d)},\dots, T_{\pi^i(D)}.$ Thus $\ell_i$ has a non-zero coefficient where all the $\ell_j,j<i$ have a zero coefficient, and thus $\ell_i$ is linearly independent from the $\ell_j,j<i.$ In particular, it follows that the $\ell_i$
are linearly independent in $R_1$. On the other
hand, they are contained in the $D-d$-dimensional sub-$\C$-vector space
$\idS_1$ and are thus a basis of $\idS_1$, and also a generating set for the ideal $\idS.$
\end{proof}
Note that all of these proofs generalize to $k$-forms. For example, one calculates that
$$\dim_\CC \idS_k= {D+k-1 \choose k} - {d+k-1 \choose k},$$
and the triangularization strategy yields a last row which corresponds to $T_{\pi(D)}^{k-1}\ell$ with a linear polynomial $\ell\in \idS_1$

\subsection{Relation to Previous Work in Computational Algebraic Geometry}
\label{sec:exact-prwork}

In this section, we discuss how the algebraic formulation of the cumulant comparison problem given in Problem~\ref{Prob:Alg_SSA} relates to the classical problems in computational algebraic geometry.

Problem~\ref{Prob:Alg_SSA} confronts us with the following task: given polynomials $q_1,\dots, q_{m-1}$ with
special properties, compute a linear generating set for the radical ideal
$$\sqrt{\langle q_1, \ldots, q_{m-1}\rangle}=\Id(\VS (q_1,\dots, q_{m-1})).$$
Computing the radical of an ideal is a classical task in computational algebraic geometry, so our problem is a special case of radical computation of ideals, which in turn can be viewed as an instance of primary decomposition of ideals, see \cite[4.7]{Cox}.

While it has been known for long time that there exist constructive algorithms to calculate the radical of a given ideal
in polynomial rings \cite{Her26}, only in the recent decades there have been algorithms feasible for implementation in modern computer algebra systems. The best known algorithms are those of \cite{GiaTraZac88Gro}, implemented in AXIOM and REDUCE, the algorithm of \cite{EisHunVas92Dir}, implemented in Macaulay 2, the algorithm of \cite{CabConCar97Yet}, currently implemented in CoCoA, and the algorithm of \cite{KriLog91Alg} and its modification by \cite{Lap06Alg}, available in SINGULAR.

All of these algorithms have two points in common. First of all, these algorithms have computational worst case complexities which are doubly exponential in the square of the number of variables of the given polynomial ring, see \cite[section 4.]{Lap06Alg}.
Although the worst case complexities may not be approached for the problem setting described
in the current paper, these off-the-shelf algorithms do not take into account the specific properties of the ideals in question.

On the other hand, Algorithm~\ref{alg:exact_covonly} can be seen as a homogenous version of the well-known Buchberger algorithm to find a Groebner basis of the dehomogenization of $\fraks$ with respect to a degree-first order. Namely, due to our strong assumptions on $m$, or as is shown in Proposition~\ref{Prop:dehom-rad-generic} in the appendix for a more general case, the homogenous saturations of the ideal $\langle q_1,\dots, q_{m-1}\rangle = \frakm\cdot \fraks$ and the ideal $\fraks$ coincide. In particular, the dehomogenizations of the $q_i$ constitute a generating set for the dehomogenization of $\fraks$. The Buchberger algorithm now finds a reduced Groebner basis of $\fraks$ which consists of exactly $D-d$ linear polynomials. Their homogenizations then constitute a basis of homogenous linear forms of $\fraks$ itself. It can be checked that the first elimination steps which the Buchberger algorithm performs for the dehomogenizations of the $q_i$ correspond directly to the elimination steps in Algorithm~\ref{alg:exact_covonly} for their homogenous versions. So our algorithm performs similarly to the Buchberger algorithm in a noiseless setting, since both algorithms compute a reduced Groebner basis in the chosen coordinate system.

However, in our setting which stems from real data, there is a second point which is more grave and
makes the use of off-the-shelf algorithms impossible: the computability of an exact
result completely relies on the assumption that the ideals given as
input are exactly known, i.e.~a generating set of polynomials is
exactly known. This is not a problem in classical computational algebra;
however, when dealing with polynomials obtained from real data, the polynomials
come not only with numerical error, but in fact with statistical uncertainty. In
general, the classical algorithms are unable to find any solution when
confronted even with minimal noise on the otherwise exact polynomials.
Namely, when we deal with a system of equations for which over-determination is possible,
any perturbed system will be over-determined and thus have no solution. For example,
the exact intersection of $N>D+1$ linear subspaces in complex $D$-space is always empty
when they are sampled with uncertainty; this is a direct consequence of
Proposition~\ref{Thm:genintcont}, when using the assumption that the noise is generic.
However, if all those hyperplanes are nearly the same, then the result of a meaningful
approximate algorithm should be a hyperplane close to all input hyperplanes
instead of the empty set.

Before we continue, we would like to stress a conceptual point in approaching uncertainty. First, as in classical numerics, one can think of the input as theoretically exact, but with fixed error $\varepsilon$ and then derive bounds on the output error in terms of this $\varepsilon$ and analyze their asymptotics. We will refer to this approach as {\it numerical uncertainty}, as opposed to {\it statistical uncertainty}, which is a view more common to statistics and machine learning, as it is more natural for noisy data. Here, the error is considered as inherently probabilistic due to small sample effects or noise fluctuation, and algorithms may be analyzed for their statistical properties, independent of whether they are themselves deterministic or stochastic. The statistical view on uncertainty is the one the reader should have in mind when reading this paper.

Parts of the algebra community have been committed to the numerical viewpoint on uncertain polynomials: the problem of numerical uncertainty is for example extensively addressed in Stetter's standard book on numerical algebra \citep{Ste04}. The main difficulties and innovations stem from the fact that standard methods from algebra like the application of Groebner bases are numerically unstable, see \cite[chapter 4.1-2]{Ste04}.

Recently, the algebraic geometry community has developed an increasing interest in solving algebraic problems arising from the consideration of real world data. The algorithms in this area are more motivated to perform well on the data, some authors start to adapt a statistical viewpoint on uncertainty, while the influence of the numerical view is still dominant. As a distinction, the authors describe the field as approximate algebra instead of numerical algebra. Recent developments in this sense can be found for example in \citep{Hel06} or the book of \cite{KrePouRob09}. We will refer to this viewpoint as the statistical view in order to avoid confusion with other meanings of approximate.

Interestingly, there are significant similarities on the methodological side. Namely, in computational algebra, algorithms often compute primarily over vector spaces, which arise for example as spaces of polynomials with certain properties. Here, numerical linear algebra can provide many techniques of enforcing numerical stability, see the pioneering paper of \cite{Cor95}. Since then, many algorithms in numerical and approximate algebra utilize linear optimization to estimate vector spaces of polynomials. In particular, least-squares-approximations
of rank or kernel are canonical concepts in both numerical and approximate algebra.

However, to the best of our knowledge, there is to date no algorithm which computes an ``approximate'' (or ``numerical'') radical of an ideal, or an approximate saturation, and also none in our special case. In the next section, we will use estimation techniques from linear algebra to convert Algorithm~\ref{alg:exact_covonly} into an algorithm which can cope with the inherent statistical uncertainty of the estimation problem.

\section{Approximate Algebraic Geometry on Real Data}
\label{sec:approx}
\label{sec:assa-appr}

In this section we show how algebraic computations can be applied to polynomials with inexact
coefficients obtained from estimated cumulants on finite samples. Note that our method for
computing the approximate radical is not specific to the problem studied in this paper.


The reason why we cannot directly apply our algorithm for the exact case to estimated polynomials
is that it relies on the assumption that there exists an exact solution, such that
the projected cumulants are equal, i.e.~we can find a projection $P$ such that the equalities
\begin{align*}
	P \Sigma_1  P^\top = \cdots = P \Sigma_m P^\top  \hspace{0.3cm} \text{and} \hspace{0.3cm}
	 P  \mu_1 = \cdots = P \mu_m
\end{align*}
hold exactly. However, when the elements of $\Sigma_1, \ldots, \Sigma_m$ and $\mu_1, \ldots, \mu_m$ are subject
to random fluctuations or noise, there exists no projection that yields exactly the same random variables.
%
In algebraic terms, working with inexact polynomials means that the
joint vanishing set of $q_1, \ldots, q_{m-1}$ and $f_1, \ldots, f_{m-1}$
consists only of the origin $0 \in \C^D$ so that the ideal becomes
trivial:
$$\langle q_1,\dots, q_{m-1}, f_1,\dots, f_{m-1}\rangle = \frakm.$$
Thus, in order to find a meaningful solution,
we need to compute the radical approximately.


In the exact algorithm, we are looking for a polynomial of the form $T_D\ell$ vanishing on $S$, which is also a $\CC$-linear combination of the quadratic forms $q_i.$ The algorithm is based on an explicit way to do so which works since the $q_i$ are generic and sufficient in number. So one could proceed to adapt this algorithm to the approximate case by performing the same operations as in the exact case and then taking the $(\nD-\nd)$-th row, setting coefficients not divisible by $T_D$ to zero, and then dividing out $T_D$ to get a linear form.
This strategy performs fairly well for small dimensions $D$ and converges to the correct solution, albeit slowly.

Instead of computing one particular linear generator as in the exact case, it is advisable to utilize as much information as possible in order to obtain better accuracy. The least-squares-optimal way to approximate a linear space of known dimension is to use singular value decomposition (SVD): with this method, we may directly eliminate the most insignificant directions in coefficient space which are due to fluctuations in the input.
To that end, we first define an approximation of an arbitrary matrix by a matrix of fixed rank.
\begin{Def}
Let $A \in \C^{m \times n}$ with singular value decomposition
$A = U D V^*,$
where $D =\diag (\sigma_1,\dots, \sigma_p)\in \C^{p \times p}$ is a diagonal matrix with ordered singular values on the diagonal,
\begin{align*}
	|\sigma_{1}| \geq |\sigma_{2}| \geq \cdots \geq |\sigma_{p}| \geq 0.
\end{align*}
For $k\le p,$ let
$D'=\diag (\sigma_1,\dots, \sigma_k, 0,\dots, 0).$
Then the matrix
$A' = U D' V^*$
is called {\it rank $k$ approximation} of $A.$
The null space, left null space, row span, column span of $A'$ will be called {\it rank $k$ approximate null space, left null space, row span, column span} of $A.$
\end{Def}
For example, if $u_1, \ldots, u_p$ and $v_1, \ldots, v_p$ are the columns of $U$ and $V$ respectively, the rank $k$ approximate left null space of $A$ is spanned by the rows of the matrix
\begin{align*}
	L = \begin{bmatrix}
				u_{p-k+1} & \cdots &
				u_p			
		\end{bmatrix}^\top ,
\end{align*}
and the rank $k$ approximate row span of $A$ is spanned by the rows of the matrix
\begin{align*}
	S = \begin{bmatrix}
				v_{1} & \cdots &
				v_p			
		\end{bmatrix}^\top .
\end{align*}
We will call those matrices the {\it approximate left null space matrix} resp.~the {approximate row span matrix \it} of rank $k$ associated to $A.$ The approximate matrices are the optimal approximations of rank $k$ with respect to the least-squares error.

We can now use these concepts to obtain an approximative version of Algorithm~\ref{alg:exact_covonly}. Instead of searching for a single element of the form $T_D\ell,$ we estimate the vector space of all such elements via singular value decomposition --- note that this is exactly the vector space $\left(\langle T_D\rangle \cdot \fraks\right)_2$, i.e.~the vector space of all homogenous polynomials of degree two which are divisible by $T_D$. Also note that the choice of the linear form $T_D$ is irrelevant, i.e.\ we may replace $T_D$ above by any variable or even linear form. As a trade-off between accuracy and runtime, we additionally estimate the vector spaces $\left(\langle T_D\rangle \cdot \fraks\right)_2$ for all $1\le i\le D$, and then least-squares average the putative results for $\fraks$ to obtain a final estimator for $\fraks$ and thus the desired space of projections.

\begin{algorithm}[h]
\caption{\label{alg:approx_covonly}
The \textit{input} consists of noisy
quadratic forms $q_1, \ldots, q_{m-1} \in \C[T_1, \ldots, T_D]$, and the dimension $d$;
the \textit{output} is an approximate linear generating set
$\ls_1, \ldots, \ls_{D-d}$ for the ideal $\idS$. }

\begin{algorithmic}[1]

	\State Let $Q \gets \begin{bmatrix} q_1 & \cdots & q_{m-1} \end{bmatrix}^\top$ be the $(m-1 \times \nD)$-matrix
		of coefficient vectors, where every row corresponds to a polynomial and every column to a monomial $T_i T_j$
		in arbitrary order.
		
	\For{$i=1,\ldots,D$}
		\State
		\begin{minipage}[t]{14cm}
			Let $Q_i$ be the $((m-1) \times \nD - D)$-sub-matrix of $Q$ obtained by removing all
    		columns corresponding to monomials divisible by $T_i$

		\end{minipage}
		
		\State
		\begin{minipage}[t]{14cm}
			Compute the approximate left null space matrix $L_i$ of $Q_i$ of rank  $(m-1) -\nD+ \nd +D-d$
		\end{minipage}
				
		\State Compute the approximate row span matrix $L'_i$ of $L_i Q$ of rank $D-d$	
		
		\State
		\begin{minipage}[t]{14cm}
		Let $L''_i$ be the $(D-d \times D)$-matrix obtained from $L'_i$ by removing all columns corresponding to	
				monomials not divisible by $T_i$
		\end{minipage}
	\EndFor

     \State Let $L$ be the $(D(D-d)\times D)$-matrix obtained by vertical concatenation of $L''_1, \ldots, L''_D$

    	 \State Compute the approximate row span matrix
	 			$A = \begin{bmatrix} a_1 & \cdots & a_{D-d} \end{bmatrix}^\top$ of $L$ of rank $D-d$
  			and let
			$\ls_i = \begin{bmatrix} T_1 & \cdots & T_D \end{bmatrix} a_i $ for all $1 \leq i \leq D-d$.
  	  	
\end{algorithmic}
\end{algorithm}
We explain the logic behind the single steps: In the first step, we start with the same matrix $Q$ as in Algorithm 1. Instead of bringing $Q$ into triangular form with respect to the term order $T_1 \prec\dots\prec T_D,$ we compute the left kernel space row matrix $S_i$ of the monomials not divisible by $T_i$. Its left image $L_i=S_i Q$ is a matrix whose row space generates the space of possible last rows after bringing $Q$ into triangular form in an arbitrary coordinate system. In the next step, we perform PCA to estimate a basis for the so-obtained vector space of quadratic forms of type $T_i$ times linear form, and extract a basis for the vector space of linear forms estimated via $L_i.$ Now we can put together all $L_i$ and again perform PCA to obtain a more exact and numerically more estimate for the projection in the last step. The rank of the matrices after PCA is always chosen to match the correct ranks in the exact case.

Note that Algorithm~\ref{alg:approx_covonly} is a consistent estimator for the correct space of projections if the covariances are sample estimates. Let us first clarify in which sense consistent is meant here: If each covariance matrix is estimated from a sample of size $N$ or greater, and $N$ goes to infinity, then the estimate of the projection converges in probability to the true projection. The reason why Algorithm~\ref{alg:approx_covonly} gives a consistent estimator in this sense is elementary: covariance matrices can be estimated consistently, and so can their differences, the polynomials $q_i$. Moreover, the algorithm can be regarded as an almost continuous function in the polynomials $q_i$; so convergence in probability to the true projection and thus consistency follows from the continuous mapping theorem. 

The runtime complexity of Algorithm~\ref{alg:approx_covonly} is $O(D^6)$ as for Algorithm~\ref{alg:exact_covonly}. For this note that calculating the singular value decomposition of an $m\times n$-matrix is $O(mn\max (m,n)).$

If we want to consider $k$-forms instead of $2$-forms, we can use the same strategies as above to numerically stabilize the exact algorithm. In the second step, one might want to consider all sub-matrices $Q_M$ of $Q$ obtained by removing all columns corresponding to monomials divisible by some degree $(k-1)$ monomial $M$ and perform the for-loop over all such monomials or a selection of them. Considering $D$ monomials or more gives again a consistent estimator for the projection. Similarly, these methods allow us to numerically stabilize versions with reduced epoch requirements and simultaneous consideration of different degrees.

\section{Numerical Evaluation}
\label{sec:sims}

In this section we evaluate the performance of the algebraic algorithm on
synthetic data in various settings. In order to contrast the algebraic approach with
an optimization-based method (cf.~Figure~\ref{fig:ml_optim}), we compare
with the Stationary Subspace Analysis (SSA) algorithm~\citep{PRL:SSA:2009}, which
solves a similar problem in the context of time series analysis. To date, SSA has been successfully applied in the context of
		biomedical data analysis~\citep{BunMeiSchMul10Finding},
		domain adaptation~\citep{HarKawWasBun10SSA}, change-point detection~\citep{BunMeiSchMul10Boosting} 
and computer vision~\citep{MeiBunKawMul09Learning}.

\subsection{Stationary Subspace Analysis}
\label{sec:sims-SSA}

Stationary Subspace Analysis~\citep{PRL:SSA:2009, MulBunMeiKirMul11SSAToolbox} factorizes an observed time series
according to a linear model into underlying stationary and non-stationary sources.
The observed time series $x(t) \in \R^D$ is assumed to
be generated as a linear mixture of stationary sources $s^\s(t) \in \R^d$ and non-stationary sources
$s^\n(t) \in \R^{D-d}$,
 \begin{align}
  x(t) = A s(t) = \begin{bmatrix} A^{\s} & A^{\n} \end{bmatrix}
  \begin{bmatrix} s^{\s}(t) \\ s^{\n}(t) \end{bmatrix} ,
\label{eq:mixing_model}
\end{align}
with a time-constant mixing matrix $A$. The underlying sources $s(t)$ are not assumed to be
independent or uncorrelated.

The aim of SSA is to invert this mixing model given only samples from $x(t)$. The true mixing
matrix $A$ is not identifiable~\citep{PRL:SSA:2009}; only the projection $P \in \R^{d \times D}$
to the stationary sources can be estimated from the mixed signals $x(t)$, up to arbitrary
linear transformation of its image. The estimated stationary sources are given by
$\hat{s}^\s(t) = P x(t)$, i.e.~the projection $P$ eliminates all non-stationary contributions: $P A^\n = 0$.

The SSA algorithms~\citep{PRL:SSA:2009,HarKawWasBun10SSA} are based on the following
definition of stationarity: a time series $X_t$ is considered stationary if its mean
and covariance is constant over time, i.e.~$\E[X_{t_1}] = \E[X_{t_2}]$ and
$\E[X_{t_1} X_{t_1}^\top] = \E[X_{t_2} X_{t_2}^\top]$ for all pairs of time points $t_1, t_2 \in \N$.
Following this concept of stationarity, the projection $P$ is found by minimizing
the difference between the first two moments of the estimated stationary sources
$\hat{s}^\s(t)$ across epochs of the times series. To that end, the samples from $x(t)$ are divided
into $m$ non-overlapping epochs of equal size, corresponding to the index sets
$\mathcal{T}_1, \ldots, \mathcal{T}_{m}$, from which the
mean and the covariance matrix is estimated for all epochs $1 \leq i \leq m$,
\begin{align*}
\emu_i  = \frac{1}{|\mathcal{T}_i|} \sum_{t \in \mathcal{T}_i} x(t) \hspace{0.5cm} \text{ and }  \hspace{0.5cm}
\esi_i  = \frac{1}{|\mathcal{T}_i|-1} \sum_{t \in \mathcal{T}_i} \left( x(t)-\emu_i \right)\left( x(t)-\emu_i \right)^\top .
\end{align*}
Given a projection $P$, the mean and the covariance of the estimated stationary sources
in the $i$-th epoch are given by $\emu^\s_i = P \emu_i$ and $\esi^\s_i = P \esi_i P^\top$
respectively. Without loss of generality (by centering and
whitening\footnote{A whitening transformation is a basis transformation $W$ that sets the sample
covariance matrix to the identity. It can be obtained from the sample covariance matrix $\hat{\Sigma}$ as
$W = \hat{\Sigma}^{-\frac{1}{2}}$} the average epoch) we
can assume that $\hat{s}^\s(t)$ has zero mean and unit covariance.

The objective function of the SSA algorithm~\citep{PRL:SSA:2009} minimizes the sum of the
differences between each epoch and the standard normal distribution, measured by the
Kullback-Leibler divergence $\KLD$ between Gaussians: the projection $P^*$ is found
as the solution to the optimization problem,
\begin{align*}
	P^* & =
	\argmin_{P P^\top = I} \; \sum_{i=1}^{m} \KLD \Big[ \Gauss(\emu^\s_i,\esi^\s_i) \; \Big|\Big| \; \Gauss(0,I) \Big] \\
	& = \argmin_{P P^\top = I} \; \sum_{i=1}^{m} \left(
			- \log\det\esi^\s_i
			+ (\hat{\mu}^\s_i)^\top \emu^\s_i 	\right),
\end{align*}
which is non-convex and solved using an iterative gradient-based procedure.

This SSA algorithm considers a problem that is closely related to the one addressed
in this paper, because the underlying definition of stationarity does not consider
the time structure.  In essence, the $m$ epochs are modeled as $m$ random variables
$X_1, \ldots, X_{m}$ for which we want to find a projection $P$ such that
the projected probability distributions $P X_1, \ldots, P X_{m}$ are equal, up
to the first two moments. This problem statement is equivalent to the task that
we solve algebraically.

\subsection{Results}
\label{sec:sims-exps}

In our simulations, we investigate the influence of the noise level and the number of
dimensions on the performance and the runtime of our algebraic algorithm
and the SSA algorithm. We measure the performance using the subspace angle between
the true and the estimated space of projections $S$.

The setup of the synthetic data is as follows: we fix the total number of dimensions to $D=10$ and
vary the dimension $d$ of the subspace with equal probability distribution from
one to nine. We also fix the number of random variables to $m=110$. For each trial
of the simulation, we need to choose a random basis for the two subspaces $\R^D = S \oplus S^\perp$,
and for each random variable, we need to choose a covariance matrix that is
identical only on $S$. Moreover, for each random variable, we need to choose a positive definite disturbance matrix
(with given noise level $\sigma$), which is added to the covariance matrix to simulate the effect
of finite or noisy samples.

The elements of the basis vectors for $S$ and $S^\perp$ are drawn uniformly from the interval
$(-1, 1)$. The covariance matrix of each epoch $1 \leq i \leq m$ is obtained from Cholesky
factors with random entries drawn uniformly from $(-1, 1)$, where the first $d$ rows
remain fixed across epochs. This yields noise-free covariance matrices $C_1, \ldots, C_m \in \R^{D \times D}$
where the first $(d \times d)$-block is identical. Now for each $C_i$, we generate a random
disturbance matrix $E_i$ to obtain the final covariance matrix $$C'_i = C_i + E_i .$$ The disturbance matrix $E_i$
is determined as
$
	E_i = V_i D_i V_i^\top
$
where $V_i$ is a random orthogonal matrix, obtained as the matrix exponential of an
antisymmetric matrix with random elements and $D_i$ is a diagonal matrix of eigenvalues.
The noise level $\sigma$ is the log-determinant of the disturbance matrix $E_i$. Thus
the eigenvalues of $D_i$ are normalized such that
$$
	\frac{1}{10} \sum_{i=1}^{10} \log D_{ii} = \sigma .
$$
In the final step of the data generation, we transform the disturbed
covariance matrices $C'_1, \ldots, C'_m$ into the random basis to obtain
the cumulants $\Sigma_1, \ldots, \Sigma_m$ which are the input to our algorithm.

\begin{figure}[h]
  \begin{center}
    \includegraphics{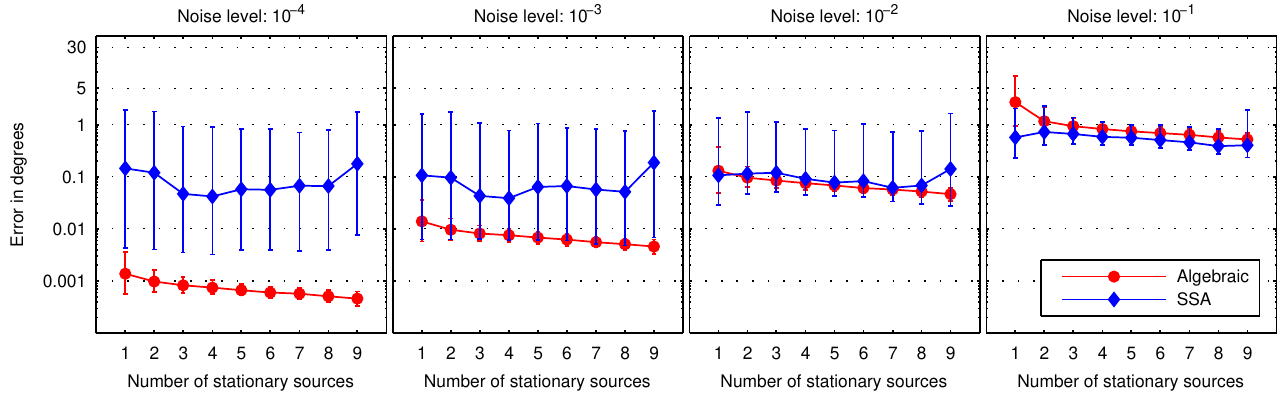}
  \caption{
        \label{fig:alg1}
        Comparison of the algebraic algorithm and the SSA algorithm. Each panel shows
        	the median error of the two algorithms (vertical axis) for varying numbers of
		stationary sources in ten dimensions (horizontal axis). The noise level increases
		from the left to the right panel; the error bars extend from the 25\% to the 75\%
		quantile estimated over 2000 random realizations of the data set.
       }
  \end{center}
\end{figure}

The first set of results is shown in Figure~\ref{fig:alg1}. With increasing noise levels
(from left to right panel) both algorithms become worse. For low noise levels, the algebraic method
yields significantly better results than the optimization-based approach, over all dimensionalities.
For medium and high-noise levels, this situation is reversed.

\begin{figure*}[h]
  \begin{center}
    \includegraphics{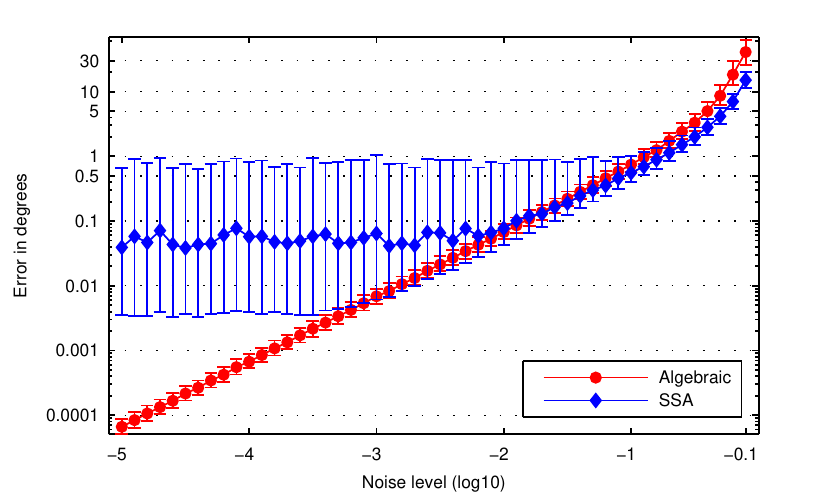}
    \includegraphics{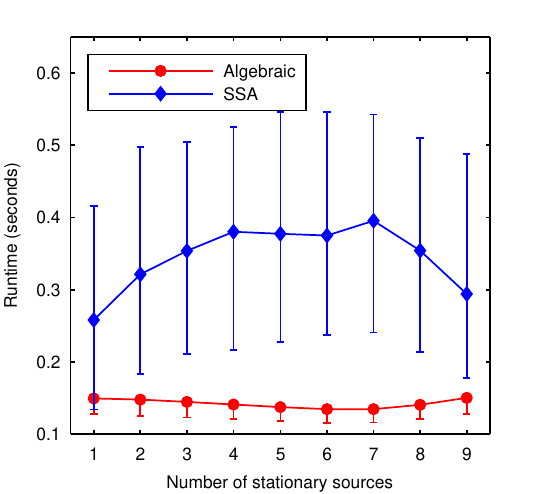}
  \caption{
        \label{fig:alg2} The left panel shows a comparison of the algebraic method and the SSA algorithm
        	over varying noise levels (five stationary sources in ten dimensions), the two curves
		show the median log error. The right panel shows a comparison of the runtime for varying
		numbers of stationary sources. The error bars extend from the 25\% to
		the 75\% quantile estimated over 2000 random realizations of the data set.
       }
  \end{center}
\end{figure*}

In the left panel of Figure~\ref{fig:alg2}, we see that the error level of the algebraic algorithm decreases with the noise level, converging to the exact solution when the noise tends to zero. In contrast, the error of original SSA decreases with noise level, reaching a minimum error baseline which it cannot fall below. In particular, the algebraic method significantly outperforms SSA for low noise levels, whereas SSA
is better for high noise. However, when noise is too high, none of the two algorithms can find the correct solution. In the right panel of Figure~\ref{fig:alg2}, we see that the algebraic method is significantly faster than SSA.

\clearpage

\section{Conclusion}
\label{sec:concl}

In this paper we have shown how a learning problem formulated in terms of
cumulants of probability distributions can be addressed in the framework of computational
algebraic geometry. As an example, we have demonstrated this viewpoint on the problem
of finding a linear map $P \in \R^{d \times D}$ such that a set of projected random
variables $X_1, \ldots, X_m \in \R^D$ have the same distribution,
\begin{align*}
	P X_1 \sim \cdots \sim P X_m .
\end{align*}
To that end, we have introduced the theoretical groundwork for an
algebraic treatment of inexact cumulants estimated from data: the concept of polynomials
that are \textit{generic} up to a certain property that we aim to recover from data.
In particular, we have shown how we can find an approximate exact solution to this
problem using algebraic manipulation of cumulants estimated on samples drawn from
$X_1, \ldots, X_m$. Therefore we have introduced the notion of computing an
\textit{approximate saturation} of an ideal that is optimal in a least-squares sense.
Moreover, using the algebraic problem formulation in terms of generic
polynomials, we have presented compact proofs for a condition on the identifiability
of the true solution.

In essence, instead of searching the surface of a non-convex objective function involving
the cumulants, the algebraic algorithm directly finds the solution by manipulating cumulant
polynomials --- which is the more natural representation of the problem. This viewpoint is
not only theoretically appealing, but conveys practical advantages
that we demonstrate in a numerical comparison to Stationary Subspace Analysis \citep{PRL:SSA:2009}:
the computational cost is significantly lower and the error converges to zero as the
noise level goes to zero. However, the algebraic algorithm requires $m \geq \nD$
random variables with distinct distributions, which is quadratic in the number of dimensions $D$.
This is due to the fact that the algebraic algorithm represents the cumulant polynomials in the vector space of coefficients.
Consequently,
the algorithm is confined to linearly combining the polynomials which describe the solution.
However, the set of solutions is also invariant under multiplication of
polynomials and polynomial division, i.e.~the algorithm does not utilize all information
contained in the polynomial equations. We conjecture that we can construct a
more efficient algorithm, if we also multiply and divide polynomials.

The theoretical and algorithmic techniques introduced in this paper can be applied
to other scenarios in machine learning, including the following examples.
\begin{itemize}
	\item \textbf{Finding properties of probability distributions.} Any inference problem that
		can be formulated in terms of polynomials, in principle,
		amenable our algebraic approach; incorporating polynomial constraints is also straightforward.
		
	\item \textbf{Approximate solutions to polynomial equations.} In machine learning,
		the problem of solving polynomial equations can e.g.~occur in the context of finding
		the solution to a constrained nonlinear optimization problem by means of setting the gradient to
		zero.
		
    \item \textbf{Conditions for identifiability.} Whenever a machine learning problem can be formulated
        in terms of polynomials, identifiability of its generative model can also be phrased in terms of algebraic geometry,
        where a wealth of proof techniques stands at disposition.


\end{itemize}

We argue for a cross-fertilization of approximate computational algebra and machine learning:
the former can benefit from the wealth of techniques for dealing with uncertainty and noisy data;
the machine learning community may find a novel framework for representing learning
problems that can be solved efficiently using symbolic manipulation.

\acks{We thank Marius Kloft and Jan Saputra M\"uller for valuable discussions. We are particularly grateful to Gert-Martin Greuel for his insightful remarks. We also thank Andreas Ziehe for proofreading the manuscript. This work has been supported by the Bernstein Cooperation (German Federal Ministry of Education and Science), Förderkennzeichen 01~GQ~0711,
and the Mathematisches Forschungsinstitut Oberwolfach (MFO). A preprint version of this manuscript has appeared as part of the Oberwolfach Preprint series \citep{Kir11}.
}

\appendix

\section{An example}
\label{app-example}
In this section, we will show by using a concrete example how the Algorithms~\ref{alg:exact_covonly} and \ref{alg:approx_covonly} work. The setup will be the similar to the example presented in the introduction. We will use the notation introduced in Section~\ref{sec:exact}.
\begin{Ex}\rm
In this example, let us consider the simplest non-trivial case: Two random variables $X_1,X_2$ in $\R^2$ such that there is exactly one direction $w\in \R^2$ such that $w^\top X_1=w^\top X_2$. I.e.~the total number of dimensions is $D=2$, the dimension of the set of projections is $d=1$. As in the beginning of Section~\ref{sec:exact}, we may assume that $\R^2=S\oplus S^\perp$ is an orthogonal sum of a one-dimensional space of projections $S$ and its orthogonal complement $S^\perp$. In particular, $S^\perp$ is given as the linear span of a single vector, say $\begin{bmatrix} \alpha & \beta \end{bmatrix}^\top$. The space $S$ is also the linear span of the vector $\begin{bmatrix} \beta & -\alpha \end{bmatrix}^\top.$

Now we partition the sample into $D(D+1)/2-d(d+1)/2=2$ epochs (this is the lower bound needed by Proposition~\ref{Prop:Alg1corr}). From the two epochs we can estimate two covariance matrices $\hat{\Sigma}_1,\hat{\Sigma}_2$. Suppose we have
\begin{align}
 \hat{\Sigma}_1	   & = \begin{bmatrix}
		   				a_{11} & a_{12} \\ a_{21} & a_{22}
		   		      \end{bmatrix} .
\end{align}
From this matrices, we can now obtain a polynomial
\begin{align}
	q_1 & = w^\top (\hat{\Sigma}_1 - I) w \nonumber \\
		   & = w^\top \begin{bmatrix}
		   				a_{11}-1 & a_{12} \\ a_{21} & a_{22}-1
		   		      \end{bmatrix} w \nonumber \\
		   & = (a_{11}-1) T_1^2 + ( a_{12} + a_{21}) T_1 T_2 + (a_{22}-1) T_2^2,
\end{align}
where $w=\begin{bmatrix} T_1 & T_2 \end{bmatrix}^\top.$ Similarly, we obtain a polynomial $q_2$ as the Gram polynomial of $\hat{\Sigma}_2-I.$

First we now illustrate how Algorithm~\ref{alg:exact_covonly}, which works with homogenous exact polynomials, can determine the vector space $S$ from these polynomials. For this, we assume that the estimated polynomials are exact; we will discuss the approximate case later. We can also write $q_1$ and $q_2$ in coefficient expansion:
\begin{align*}
q_1&=q_{11}T_1^2+q_{12}T_1T_2+q_{13}T_2^2\\
q_2&=q_{21}T_1^2+q_{22}T_1T_2+q_{23}T_2^2.
\end{align*}
We can also write this formally in the $(2\times 3)$ coefficient matrix $Q=(q_{ij})_{ij},$ where the polynomials can be reconstructed as the entries in the vector
$$Q\cdot \begin{bmatrix} T_1^2 & T_1T_2 & T_2^2 \end{bmatrix}^\top.$$
Algorithm~\ref{alg:exact_covonly} now calculates the upper triangular form of this matrix. For polynomials, this is equivalent to calculating the last row
\begin{align*}
 &q_{21} q_1 -q_{11} q_2\\
&=[q_{21}q_{12}-q_{11}q_{22}]T_1T_2+[q_{21}q_{13}-q_{11}q_{23}]T_2^2.
\end{align*}
Then we divide out $T_2$ and obtain
$$P=[q_{21}q_{12}-q_{11}q_{22}]T_1+[q_{21}q_{13}-q_{11}q_{23}]T_2.$$
The algorithm now identifies $S^\perp$ as the vector space spanned by the vector
$$\begin{bmatrix} \alpha & \beta \end{bmatrix}^\top
=\begin{bmatrix}q_{21}q_{12}-q_{11}q_{22} & q_{21}q_{13}-q_{11}q_{23}\end{bmatrix}^\top.$$
This already finishes the calculation given by Algorithm~\ref{alg:exact_covonly}, as we now explicitly know the solution
$$\begin{bmatrix} \alpha & \beta \end{bmatrix}^\top$$
To understand why this strategy works, we need to have a look at the input. Namely, one has to note that $q_1$ and $q_2$ are generic homogenous polynomials of degree $2$, vanishing on $S$. That is, we will have $q_i(x)=0$ for $i=1,2$ and all points $x\in S.$ It is not difficult to see that every polynomial fulfilling this condition has to be of the form
$$(\alpha T_1+\beta T_2)(a T_1 +bT_2)$$
for some $a,b\in \C;$ i.e.~a multiple of the equation defining $S$. However we may not know this factorization a priori, in particular we are in general agnostic as to the correct values of $\alpha$ and $\beta$. They have to be reconstructed from the $q_i$ via an algorithm. Nonetheless, a correct solution exists, so we may write
\begin{align*}
q_1&=(\alpha T_1+\beta T_2)(a_1 X +b_1 T_2)\\
q_2&=(\alpha T_1+\beta T_2)(a_2 X +b_2 T_2),
\end{align*}
with $a_i,b_i$ generic, without knowing the exact values a priori. If we now compare to the above expansion in the $q_{ij},$ we obtain the linear system of equations
\begin{align*}
q_{i1}&=\alpha a_i\\
q_{i2}&=\alpha b_i+ \beta a_i\\
q_{i3}&=\beta b_i
\end{align*}
for $i=1,2$, from which we may reconstruct the $a_i,b_i$ and thus $\alpha$ and $\beta$. However, a more elegant and general way of getting to the solution is to bring the matrix $Q$ as above into triangular form. Namely, by assumption, the last row of this triangular form corresponds to the polynomial $P$ which vanishes on $S$. Using the same reasoning as above, the polynomial $P$ has to be a multiple of $(\alpha T_1+\beta T_2)$. To check the correctness of the solution, we substitute the $q_{ij}$ in the expansion of $P$ for $a_i,b_i$, and obtain
\begin{align*}
P=&[q_{21}q_{12}-q_{11}q_{22}]T_1T_2+[q_{21}q_{13}-q_{11}q_{23}]T_2^2\\
=&[\alpha a_2 (\alpha b_1+\beta a_1)-\alpha a_1 (\alpha b_2+\beta a_2)]T_1T_2+[\alpha a_2 \beta b_1 -\alpha a_1 \beta b_2]T_2^2\\
=&[\alpha^2 a_2 b_1-\alpha^2 a_1 b_2]T_1T_2+[\alpha\beta a_2 b_1 -\alpha\beta a_1b_2]T_2^2\\
=&(\alpha T_1 + \beta T_2) \alpha [a_2b_1-a_1b_2] T_2.
\end{align*}
This is $(\alpha T_1+\beta T_2)$ times $T_2$ up to a scalar multiple - from the coefficients of the form $P$, we may thus directly reconstruct the vector $\begin{bmatrix} \alpha & \beta \end{bmatrix}$ up to a common factor and thus obtain a representation for $S,$ since the calculation of these coefficients did not depend on a priori knowledge about $S.$

If the estimation of the $\hat{\Sigma}_i$ and thus of the $q_i$ is now endowed with noise, and we have more than two epochs and polynomials, Algorithm~\ref{alg:approx_covonly} provides the possibility to perform this calculation approximately. Namely, Algorithm~\ref{alg:approx_covonly} finds a linear combination of the $q_i$ which is approximately of the form $T_D\ell$ with a linear form $\ell$ in the variables $T_1,T_2$. The Young-Eckart Theorem guarantees that we obtain a consistent and least-squares-optimal estimator for $P$, similarly to the exact case. The reader is invited to check this by hand as an exercise. 
\end{Ex}

Now the observant reader may object that we may have simply obtained the linear form $(\alpha T_1+\beta T_2)$ and thus $S$ directly from factoring $q_1$ and $q_2$ and taking the unique common factor. This is true, but this strategy can only be applied in the very special case $D-d=1.$ To illustrate the additional difficulties in the general case, we repeat the above example for $D=4$ and $d=2$ for the exact case: 
\begin{Ex}\rm
In this example, we need already $D(D+1)/2-d(d+1)/2=7$ polynomials $q_1,\dots, q_7$ to solve the problem with Algorithm~\ref{alg:exact_covonly}. As above, we can write
\begin{align*}
q_i=&q_{i1}T_1^2+q_{i2}T_1T_2+q_{i3}T_1T_3+q_{i4}T_1T_4+q_{i5}T_2^2\\
&+q_{i6}T_2T_3+q_{i7}T_2T_4+q_{i8}T_3^2+q_{i9}T_3T_4+q_{i,10}T_4^2\\
\end{align*}
for $i=1,\dots, 7$, and again we can write this in a $(7\times 10)$ coefficient matrix $Q=(q_{ij})_{ij}$. In Algorithm~\ref{alg:exact_covonly}, this matrix is brought into triangular form. The last row of this triangular matrix will thus correspond to a polynomial of the form
$$P =p_7T_2T_4+p_{8}T_3^2+p_{9}T_3T_4+p_{10}T_4^2$$
A polynomial of this form is not divisible by $T_4$ in general. However, Proposition~\ref{Prop:Alg1corr} guarantees us that the coefficient $p_8$ is always zero due to our assumptions. So we can divide out $T_4$ to obtain a linear form
$$p_7T_2+p_{9}T_3+p_{10}T_4.$$
This is one equation defining the linear space $S$. One obtains another equation in the variables $T_1,T_2,T_3$ if one, for example, inverts the numbering of the variables $1-2-3-4$ to $4-3-2-1$. Two equations suffice to describe $S$, and so Algorithm~\ref{alg:exact_covonly} yields the correct solution.

As in the example before, it can be checked by hand that the coefficient $p_7$ indeed vanishes, and the obtained linear equations define the linear subspace $S$. For this, one has to use the classical result from algebraic geometry that every $q_i$ can be written as
$$q_i=\ell_1 P_1+\ell_2 P_2,$$
where the $\ell_i$ are fixed but arbitrary linear forms defining $S$ as their common zero set, and the $P_i$ are some linear forms determined by $q_i$ and the $\ell_i$ (this is for example a direct consequence of Hilbert's Nullstellensatz). Caution is advised as the equations involved become very lengthy - while not too complex - already in this simple example. So the reader may want to check only that the coefficient $p_8$ vanishes as claimed.

\end{Ex}

\section{Algebraic Geometry of Genericity}
\label{app-generic}

In the paper, we have reformulated a problem of comparing probability distributions in algebraic terms. For the problem to be well-defined, we need the concept of genericity for the cumulants. The solution can then be determined as an ideal generated by generic homogenous polynomials vanishing on a linear subspace. In this supplement, we will extensively describe this property which we call genericity and derive some simple consequences.

Since genericity is an algebraic-geometric concept, knowledge about basic algebraic geometry will be required for an understanding of this section. In particular, the reader should be at least familiar with the following concepts before reading this section: Polynomial rings, ideals, radicals, factor rings, algebraic sets, algebra-geometry correspondence (including Hilbert's Nullstellensatz), primary decomposition, height resp.~dimension theory in rings. A good introduction into the necessary framework can be found in the book of \cite{Cox}.

\label{sec:Alg-Generic}
\subsection{Definition of genericity}
\label{sec:gendef}
In the algebraic setting of the paper, we would like to calculate the radical of an ideal
$$\calI=\langle q_1,\dots, q_{m-1}, f_1,\dots, f_{m-1}\rangle.$$
This ideal $\calI$ is of a special kind: its generators are random, and are only subject to the constraints that they vanish on the linear subspace $S$ to which we project, and that they are homogenous of fixed degree. In order to derive meaningful results on how $\calI$ relates to $S$, or on the solvability of the problem, we need to model this kind of randomness.

In this section, we introduce a concept called genericity. Informally, a generic situation is a situation without pathological degeneracies. In our case, it is reasonable to believe that apart from the conditions of homogeneity and the vanishing on $S$, there are no additional degeneracies in the choice of the generators. So, informally spoken, the ideal $\calI$ is generated by generic homogenous elements vanishing on $S.$ This section is devoted to developing a formal theory in order to address such generic situations efficiently.

The concept of genericity is already widely used in theoretical computer science, combinatorics or discrete mathematics; there, it is however often defined inexactly or not at all, or it is only given as an ad-hoc definition for the particular problem. On the other hand, genericity is a classical concept in algebraic geometry, in particular in the theory of moduli. The interpretation of generic properties as probability-one-properties is also a known concept in applied algebraic geometry, e.g.~algebraic statistics. However, the application of probability distributions and genericity to the setting of generic ideals, in particular in the context of conditional probabilities, are original to the best of our knowledge, though not being the first one to involve generic resp.~general polynomials, see \citep{Iar84}. Generic polynomials and ideals have been also studied in \citep{Fro94}. A collection of results on generic polynomials and ideals which partly overlap with ours may also be found in the recent paper \citep{Par10}.

Before continuing to the definitions, let us explain what genericity should mean. Intuitively, generic objects are objects without unexpected pathologies or degeneracies. For example, if one studies say $n$ lines in the real plane, one wants to exclude pathological cases where lines lie on each other or where many lines intersect in one point. Having those cases excluded means examining the ``generic'' case, i.e. the case where there are $n(n+1)/2$ intersections, $n(n+1)$ line segments and so forth. Or when one has $n$ points in the plane, one wants to exclude the pathological cases where for example there are three affinely dependent points, or where there are more sophisticated algebraic dependencies between the points which one wants to exclude, depending on the problem.

In the points example, it is straightforward how one can define genericity in terms of sampling from a probability distribution:
one could draw the points under a suitable continuous probability distribution from real two-space. Then, saying that the points are ``generic'' just amounts to examine properties which are true with probability one for the $n$ points. Affine dependencies for example would then occur with probability zero and are automatically excluded from our interest. One can generalize this idea to the lines example: one can parameterize the lines by a parameter space, which in this case is two-dimensional (slope and ordinate), and then sample lines uniformly distributed in this space (one has of course to make clear what this means).  For example, lines lying on each other or more than two lines intersecting at a point would occur with probability zero, since the part of parameter space for this situation would have measure zero under the given probability distribution.

When we work with polynomials and ideals, the situation gets a bit more complicated, but the idea is the same. Polynomials are uniquely determined by their coefficients, so they can naturally be considered as objects in the vector space of their coefficients. Similarly, an ideal can be specified by giving the coefficients of some set of generators. Let us make this more explicit: suppose first we have given a single polynomial $f\in \C[X_1,\dots X_D]$ of degree $k$.

In multi-index notation, we can write this polynomial as a finite sum
$$f=\sum_{\alpha\in \NN^D}c_\alpha X^\alpha\,\quad \mbox{with}\;c_\alpha\in \C.$$
This means that the possible choices for $f$ can be parameterized by the ${D+k \choose k}$ coefficients $c_I$ with $\|I\|_1\le k.$ Thus polynomials of degree $k$ with complex coefficients can be parameterized by complex ${D+k \choose k}$-space.

Algebraic sets can be similarly parameterized by parameterizing the generators of the corresponding ideal. However, this correspondence is highly non-unique, as different generators may give rise to the same zero set. While the parameter space can be made unique by dividing out redundancies, which gives rise to the Hilbert scheme, we will instead use the redundant, though pragmatic characterization in terms of a finite dimensional vector space over $\C$ of the correct dimension.

We will now fix notation for the parameter space of polynomials and endow it with algebraic structure. The extension to ideals will then be derived later. Let us write $\calM_k$ for complex ${D+k \choose k}$-space (we assume $D$ as fixed), interpreting it as a parameter space for the polynomials of degree $k$ as shown above. Since the parameter space $\calM_k$ is isomorphic to complex ${D+k\choose k}$-space, we may speak about algebraic sets in $\calM_k$. Also, $\calM_k$ carries the complex topology induced by the topology on $\R^{2k}$ and by topological isomorphy the Lebesgue measure; thus it also makes sense to speak about probability distributions and random variables on $\calM_k.$ This dual interpretation will be the main ingredient in our definition of genericity, and will allow us to relate algebraic results on genericity to the probabilistic setting in the applications. As $\calM_k$ is a topological space, we may view any algebraic set in $\calM_k$ as an event if we randomly choose a polynomial in $\calM_k$:
\begin{Def}
Let $X$ be a random variable with values in $\calM_k$. Then an event for $X$ is called {\it algebraic event} or {\it algebraic property} if the corresponding event set in $\calM_k$ is an algebraic set. It is called {\it irreducible} if the corresponding event set in $\calM_k$ is an irreducible algebraic set.
\end{Def}

If an event $A$ is irreducible, this means that if we write $A$ as the event ``$A_1$ and $A_2$'', for algebraic events $A_1,A_2$, then $A=A_1$, or $A=A_2.$ We now give some examples for algebraic properties.

\begin{Ex}\rm\label{Ex:algevts}
The following events on $\calM_k$ are algebraic:
\begin{enumerate}
  \item The sure event.
  \item The empty event.
  \item The polynomial is of degree $n$ or less.
  \item The polynomial vanishes on a prescribed algebraic set.
  \item The polynomial is contained in a prescribed ideal.
  \item The polynomial is homogenous.
  \item The polynomial is a square.
  \item The polynomial is reducible.
\end{enumerate}
Properties 1-5 are additionally irreducible.

We now show how to prove these claims: 1-2 are clear, we first prove that properties 3-5 are algebraic and irreducible. By definition, it suffices to prove that the subset of $\calM_k$ corresponding to those polynomials is an irreducible algebraic set. We claim: in any of those cases, the subset in question is moreover a linear subspace, and thus algebraic and irreducible. This can be easily verified by checking directly that if $f_1,f_2$ fulfill the property in question, then $f_1+\alpha f_2$ also fulfills the property.

Property 6 is algebraic, since it can be described as the disjunction of the properties ``The polynomial is homogenous and of degree $n$'' for all $n\le k.$ Those single properties can be described by linear subspaces of $\calM_k$ as above, thus property 6 is parameterized by the union of those linear subspaces. In general, these are orthogonal, so property 6 is not irreducible.

Property 7 is algebraic, as we can check it through the vanishing of a system of generalized discriminant polynomials. One can show that it is also irreducible since the subset of $\calM_k$ in question corresponds to the image of a Veronese map (homogenization to degree $k$ is a strategy); however, since we will not need such a result, we do not prove it here.

Property 8 is algebraic, since factorization can also be checked by sets of equations. One has to be careful here though, since those equations depend on the degrees of the factors. For example, a polynomial of degree $4$ may factor into two polynomials of degree $1$ and $3$, or in two polynomials of degree $2$ each. Since in general each possible combination defines different sets of equations and thus different algebraic subsets of $\calM_k$, property 8 is in general not irreducible (for $k\le 3$ it is).
\end{Ex}

The idea defining a choice of polynomial as generic follows the intuition of the affirmed non-sequitur: a generic, resp.~generically chosen polynomial should not fulfill any algebraic property. A generic polynomial, having a particular simple (i.e.~irreducible) algebraic property, should not fulfill any other algebraic property which is not logically implied by the first one. Here, algebraic properties are regarded as the natural model for restrictive and degenerate conditions, while their logical negations are consequently interpreted as generic, as we have seen in Example~\ref{Ex:algevts}. These considerations naturally lead to the following definition of genericity in a probabilistic context:

\begin{Def}\label{Def:genrand}
Let $X$ be a random variable with values in $\calM_k$. Then $X$ is called {\it generic}, if for any irreducible algebraic events $A,B,$ the following holds:

The conditional probability $P_X(A|B)$ exists and vanishes if and only if $B$ does not imply $A$.
\end{Def}
In particular, $B$ may also be the sure event.

Note that without giving a further explication, the conditional probability $P_X(A|B)$ is not well-defined, since we condition on the event $B$ which has probability zero. There is also no unique way of remedying this, as for example the Borel-Kolmogorov paradox shows. In section~\ref{sec:genalt}, we will discuss the technical notion which we adopt to ensure well-definedness.

Intuitively, our definition means that an event has probability zero to occur unless it is logically implied by the assumptions. That is, degenerate dependencies between events do not occur.

For example, non-degenerate multivariate Gaussian distributions or Gaussian mixture distributions on $\calM_k$ are generic distributions. More general, any positive continuous probability distribution which can be approximated by Gaussian mixtures is generic (see Example~\ref{Ex:gen-nongen}). Thus we argue that non-generic random variables are very pathological cases. Note however, that our intention is primarily not to analyze the behavior of particular fixed generic random variables (this is part of classical statistics). Instead, we want to infer statements which follow not from the particular structure of the probability function, but solely  from the fact that it is generic, as these statements are intrinsically implied by the conditional postulate in Definition~\ref{Def:genrand} alone. We will discuss the definition of genericity and its implications in more detail in section~\ref{sec:genalt}.

With this definition, we can introduce the terminology of a generic object: it is a generic random variable which is object-valued.

\begin{Def}
We call a generic random variable with values in $\calM_k$ a generic polynomial of degree $k.$ When the degree $k$ is arbitrary, but fixed (and still $\ge 1$), we will say that $f$ is a generic polynomial, or that $f$ is generic, if it is clear from the context that $f$ is a polynomial. If the degree $k$ is zero, we will analogously say that $f$ is a generic constant.\\

We call a set of constants or polynomials $f_1,\dots, f_m$ generic if they are generic and independent.\\

We call an ideal generic if it is generated by a set of $m$ generic polynomials.\\

We call an algebraic set generic if it is the vanishing set of a generic ideal.\\

Let $\calP$ be an algebraic property on a polynomial, a set of polynomials, an ideal, or an algebraic set (e.g.~homogenous, contained in an ideal et.). We will call a polynomial, a set of polynomials, or an ideal, a {\it generic} $\calP$ polynomial, set, or ideal, if it the conditional of a generic random variable with respect to $\calP$.\\

If $\calA$ is a statement about an object (polynomial, ideal etc), and $\calP$ an algebraic property, we will say briefly ``A generic $\calP$ object is $\calA$'' instead of saying ``A generic $\calP$ object is $\calA$ with probability one''.
\end{Def}

Note that formally, these objects are all polynomial, ideal, algebraic set etc -valued random variables. By convention, when we state something about a generic object, this will be an implicit probability-one statement. For example, when we say\\

``A generic green ideal is blue'',\\

this is an abbreviation for the by definition equivalent but more lengthy statement\\

``Let $f_1,\dots, f_m$ be independent generic random variables with values in $\calM_{k_1},\dots,\calM_{k_m}.$ If the ideal $\langle f_1,\dots, f_m\rangle$ is green, then with probability one, it is also blue - this statement is independent of the choice of the $k_i$ and the choice of which particular generic random variables we use to sample.\\

On the other hand, we will use the verb ``generic'' also as a qualifier for ``constituting generic distribution''. So for example, when we say\\

``The Z of a generic red polynomial is a generic yellow polynomial'',\\

this is an abbreviation of the statement\\

``Let $X$ be a generic random variable on $\calM_k,$ let $X'$ be the yellow conditional of $X$. Then the Z of $X'$ is the red conditional of some generic random variable - in particular this statement is independent of the choice of $k$ and the choice of $X$.''\\

It is important to note that the respective random variables will not be made explicit in the following subsections, since the statements will rely only on its property of being generic, and not on its particular structure which goes beyond being generic.\\

As an application of these concepts, we may now formulate the problem of comparing cumulants in terms of generic algebra:

\begin{Prob}\label{Prob:SSA-alg}
Let $\fraks=\Id(S)$, where $S$ is an unknown $d$-dimensional subspace of $\C^D$. Let
$$\calI=\langle f_1,\dots, f_m \rangle$$
with $f_i\in \fraks$ generic of fixed degree each (in our case, one and two), such that $\sqrt{\calI}=\fraks.$

Then determine a reduced Groebner basis (or another simple generating system) for $\fraks.$
\end{Prob}

As we will see, genericity is the right concept to model random sampling of polynomials, as we will derive special properties of the ideal $\calI$ which follow from the genericity of the $f_i$.

\subsection{Zero-measure conditionals, and relation to other types of genericity}\label{sec:genalt}
In this section, se will discuss the definition of genericity in Definition~\ref{Def:genrand} and ensure its well-definedness. Then we will invoke alternative definitions for genericity and show their relation to our probabilistic intuitive approach from section~\ref{sec:gendef}. As this section contains technical details and is not necessary for understanding the rest of the appendix, the reader may opt to skip it.

An important concept in our definition of genericity in Definition~\ref{Def:genrand} is the conditional probability $P_X(A|B)$. As $B$ is an algebraic set, its probability $P_X(B)$ is zero, so the Bayesian definition of conditional cannot apply. There are several ways to make it well-defined; in the following, we explain the Definition of conditional we use in Definition~\ref{Def:genrand}. The definition of conditional we use is one which is also often applied in this context.
\begin{Rem}\label{Rem:genmeas}
Let $X$ be a real random variable (e.g.~with values in $\calM_k$) with probability measure $\mu$. If $\mu$ is absolutely continuous, then by the theorem of Radon-Nikodym, there is a unique continuous density $p$ such that
$$\mu(U)=\int_U p\, d\lambda$$
for any Borel-measurable set $U$ and the Lebesgue measure $\lambda$. If we assume that $p$ is a continuous function, it is unique, so we may define a restricted measure $\mu_B$ on the event set of $B$ by setting
$$\nu(U)=\int_U p\, dH,$$
for Borel subsets of $U$ and the Hausdorff measure $H$ on $B$. If $\nu(B)$ is finite and non-zero, i.e.~$\nu$ is absolutely continuous with respect to $H$, then it can be renormalized to yield a conditional probability measure $\mu(.)|_B=\nu(.)/\nu(B).$ The conditional probability $P_X(A|B)$ has then to be understood as
$$P_X(A|B)=\int_{B}\mathbbm{1} (A\cap B)\,d\mu\mid_B,$$
whose existence in particular implies that the Lebesgue integrals $\nu (B)$ are all finite and non-zero.
\end{Rem}

As stated, we adopt this as the definition of conditional probability for algebraic sets $A$ and $B$. It is important to note that we have made implicit assumptions on the random variable $X$ by using the conditionals $P_X(A|B)$ in Remark~\ref{Rem:genmeas} (and especially by assuming that they exist): namely, the existence of a continuous density function and existence, finiteness, and non-vanishing of the Lebesgue integrals. Similarly, by stating Definition~\ref{Def:genrand} for genericity, we have made similar assumptions on the generic random variable $X$, which can be summarized as follows:

\begin{Ass}\label{Ass:gen}
$X$ is an absolutely continuous random variable with continuous density function $p$, and for every algebraic event $B$, the Lebesgue integrals
$$\int_B p\, dH,$$
where $H$ is the Hausdorff measure on $B$, are non-zero and finite.
\end{Ass}

This assumption implies the existence of all conditional probabilities $P_X(A|B)$ in Definition~\ref{Def:genrand}, and are also necessary in the sense that they are needed for the conditionals to be well-defined. On the other hand, if those assumptions are fulfilled for a random variable, it is  automatically generic:

\begin{Rem}\label{Prop:gen=cont}
Let $X$ be a $\calM_k$-valued random variable, fulfilling the Assumptions in~\ref{Ass:gen}. Then, the probability density function of $X$ is strictly positive. Moreover, $X$ is a generic random variable.
\end{Rem}
\begin{proof}
Let $X$ be a $\calM_k$-valued random variable fulfilling the Assumptions in~\ref{Ass:gen}. Let $p$ be its continuous probability density function.

We first show positivity: If $X$ would not be strictly positive, then $p$ would have a zero, say $x$. Taking $B=\{x\},$ the integral $\int_B p\, dH$ vanishes, contradicting the assumption.

Now we prove genericity, i.e.~that for arbitrary irreducible algebraic properties $A,B$ such that $B$ does not imply $A$, the conditional probability $P_X(A|B)$ vanishes. Since $B$ does not imply $A$, the algebraic set defined by $B$ is not contained in $A$. Moreover, as $B$ and $A$ are irreducible and algebraic, $A\cap B$ is also of positive codimension in $B$. Now by assumption, $X$ has a positive continuous probability density function $f$ which by assumption restricts to a probability density on $B$, being also positive and continuous. Thus the integral
$$P_X(A|B)=\int_B \mathbbm{1}_A f(x)\, dH,$$
where $H$ is the Hausdorff measure on $B$, exists. Moreover, it is zero, as we have derived that $A$ has positive codimension in $B$.
\end{proof}

This means that already under mild assumptions, which merely ensure well-definedness of the statement in the Definition~\ref{Def:genrand} of genericity, random variables are generic. The strongest of the comparably mild assumptions are the convergence of the conditional integrals, which allow us to renormalize the conditionals for all algebraic events. In the following example, a generic and a non-generic probability distribution are presented.

\begin{Ex}\rm\label{Ex:gen-nongen}
Gaussian distributions and Gaussian mixture distributions are generic, since for any algebraic set $B$, we have
$$\int_B \mathbbm{1}_{\mathcal{B}(t)}\, dH = O(t^{\dim B}),$$
where $\mathcal{B}(t)=\{x\in\mathbb{R}^n\;;\; \|x\|<t\}$ is the open disc with radius $t.$ Note that this particular bound is false in general and may grow arbitrarily large when we omit $B$ being algebraic, even if $B$ is a smooth manifold.
Thus $P_X(A|B)$ is bounded from above by an integral (or a sum) of the type
$$\int_{0}^\infty\exp(-t^2)t^a\;dt\quad\mbox{with}\; a\in\mathbb{N}$$
which is known to be finite.

Furthermore, sums of generic distributions are again generic; also, one can infer that any continuous probability density dominated by the distribution of a generic density defines again a generic distribution.

An example of a non-generic but smooth distribution is given by the density function
$$p(x,y)=\frac{1}{\mathcal{N}}e^{-x^4y^4}$$
where $\mathcal{N}$ is some normalizing factor. While $p$ is integrable on $\mathbb{R}^2,$ its restriction to the coordinate axes $x=0$ and $y=0$ is constant and thus not integrable.
\end{Ex}

Now we will examine different known concepts of genericity and relate them briefly to the one we have adopted. 

A definition of genericity in combinatorics and geometry which can be encountered in different variations is that there exist no degenerate interpolating functions between the objects:

\begin{Def}\label{Def:gencomb}
Let $P_1,\dots, P_m$ be points in the vector space $\mathbb{C}^n$. Then $P_1,\dots, P_m$ are general position (or generic, general) if no $n+1$ points lie on a hyperplane. Or, in a stronger version: for any $d\in\mathbb{N}$, no (possibly inhomogenous) polynomial of degree $d$ vanishes on ${n+d \choose d}+1$ different $P_i$.
\end{Def}
As $\calM_k$ is a finite dimensional $\mathbb{C}$-vector space, this definition is in principle applicable to our situation. However, this definition is deterministic, as the $P_i$ are fixed and no random variables, and thus preferable when making deterministic statements. Note that the stronger definition is equivalent to postulating general position for the points $P_1,\dots, P_m$ in any polynomial kernel feature space.

Since not lying on a hyperplane (or on a hypersurface of degree $d$) in $\mathbb{C}^n$ is a non-trivial algebraic property for any point which is added beyond the $n$-th (resp.~the ${n+d\choose d}$-th) point $P_i$ (interpreted as polynomial in $\calM_k$), our definition of genericity implies general position. This means that generic polynomials $f_1,\dots, f_m\in\calM_k$ (almost surely) have the deterministic property of being in general position as stated in Definition~\ref{Def:gencomb}. A converse is not true for two reasons: first, the $P_i$ are fixed and no random variables. Second, even if one would define genericity in terms of random variables such that the hyperplane (resp.~hypersurface) conditions are never fulfilled, there are no statements made on conditionals or algebraic properties other than containment in a hyperplane, also Lebesgue zero sets are not excluded from occuring with positive probability.

Another example where genericity classically occurs is algebraic geometry, where it is defined rather general for moduli spaces. While the exact definition may depend on the situation or the particular moduli space in question, and is also not completely consistent, in most cases, genericity is defined as follows: general, or generic, properties are properties which hold on a Zariski-open subset of an (irreducible) variety, while very generic properties hold on a countable intersection of Zariski-open subsets  (which are thus paradoxically ''less'' generic than general resp.~generic properties in the algebraic sense, as any general resp.~generic property is very generic, but the converse is not necessarily true). In our special situation, which is the affine parameter space of tuples of polynomials, these definitions can be rephrased as follows:
\begin{Def}\label{Def:gencomb}
Let $B\subseteq\mathbb{C}^k$ be an irreducible algebraic set, let $P=(f_1,\dots, f_m)$ be a tuple of polynomials, viewed as a point in the parameter space $B.$ Then a statement resp.~property $A$ of $P$ is called very generic if it holds on the complement of some countable union of algebraic sets in $B.$ A statement resp.~property $A$ of $P$ is called general (or generic) if it holds on the complement of some finite union of algebraic sets in $B.$
\end{Def}
This definition is more or less equivalent to our own; however, our definition adds the practical interpretation of generic/very generic/general properties being true with probability one, while their negations are subsequently true with probability zero. In more detail, the correspondence is as follows:
If we restrict ourselves only to algebraic properties $A$, it is equivalent to say that the property $A$ is very generic, or general for the $P$ in $B$, and to say with our original definition that a generic $P$ fulfilling $B$ is also $A$; since if $A$ is by assumption an algebraic property, it is both an algebraic set and a complement of a finite (countable) union of algebraic sets in an irreducible algebraic set, so $A$ must be equal to an irreducible component of $B$; since $B$ is irreducible, this implies equality of $A$ and $B$. On the other hand, if $A$ is an algebraic property, it is equivalent to say that the property not-$A$ is very generic, or general for the $P$ in $B$, and to say with our original definition that a generic $P$ fulfilling $B$ is not $A$ - this corresponds intuitively to the probability-zero condition $P(A|B)=0$ which states that non-generic cases do not occur. Note that by assumption, not-$A$ is then always the complement of a finite union of algebraic sets.

\subsection{Arithmetic of generic polynomials}
In this subsection, we study how generic polynomials behave under classical operations in rings and ideals. This will become important later when we study generic polynomials and ideals.\\

To introduce the reader to our notation of genericity, and since we will use the presented facts and similar notations implicitly later, we prove the following

\begin{Lem}\label{Lem:Gen-arith}
Let $f\in\C [X_1,\dots, X_D]$ be generic of degrees $k.$ Then:\\
\itboxx{i} The product $\alpha f$ is generic of degree $k$ for any fixed $\alpha\in\C\setminus \{\0\}.$\\
\itboxx{ii} The sum $f + g$ is generic of degree $k$ for any $g\in \C [X_1,\dots, X_D]$ of degree $k$ or smaller.\\
\itboxx{iii} The sum $f + g$ is generic of degree $k$ for any generic $g\in \C [X_1,\dots, X_D]$ of degree $k$ or smaller.
\end{Lem}
\begin{proof}
(i) is clear since the coefficients of $g_1$ are multiplied only by a constant. (ii) follows directly from the definitions since adding a constant $g$ only shifts the coefficients without changing genericity. (iii) follows since $f,g$ are independently sampled: if there were algebraic dependencies between the coefficients of $f+g$, then either $f$ or $g$ was not generic, or the $f,g$ are not independent, which both would be a contradiction to the assumption.
\end{proof}

Recall again what this Lemma means: for example, Lemma~\ref{Lem:Gen-arith} (i) does not say, as one could think:\\

``Let $X$ be a generic random variable with values in the vector space of degree $k$ polynomials. Then $X=\alpha X$ for any $\alpha\in \C\setminus \{0\}.$''\\

The correct translation of Lemma~\ref{Lem:Gen-arith} (i) is:\\

``Let $X$ be a generic random variable with values in the vector space of degree $k$ polynomials. Then $X'=\alpha X$ for any fixed $\alpha\in \C\setminus \{0\}$ is a generic random variable with values in the vector space of degree $k$ polynomials''\\

The other statements in Lemma~\ref{Lem:Gen-arith} have to be interpreted similarly.\\

The following remark states how genericity translates through dehomogenization:
\begin{Lem}\label{Lem:dehom}
Let $f\in\C [X_1,\dots, X_D]$ be a generic homogenous polynomial of degree $d.$ \\
Then the dehomogenization $f(X_1,\dots, X_{D-1},1)$ is a generic polynomial of degree $d$ in the polynomial ring $\C [X_1,\dots, X_{D-1}].$\\

Similarly, let $\fraks\idof \C [X_1,\dots, X_D]$ be a generic homogenous ideal. Let $f\in \fraks$ be a generic homogenous polynomial of degree $d.$ \\
Then the dehomogenization $f(X_1,\dots, X_{D-1},1)$ is a generic polynomial of degree $d$ in the dehomogenization of $\fraks.$
\end{Lem}
\begin{proof}
For the first statement, it suffices to note that the coefficients of a homogenous polynomial of degree $d$ in the variables $X_1,\dots, X_D$ are in bijection with the coefficients of a polynomial of degree $d$ in the variables $X_1,\dots, X_{D-1}$ by dehomogenization. For the second part, recall that the dehomogenization of $\fraks$ consists exactly of the dehomogenizations of elements in $\fraks.$ In particular, note that the homogenous elements of $\fraks$ of degree $d$ are in bijection to the elements of degree $d$ in the dehomogenization of $\fraks$. The claims then follows from the definition of genericity.
\end{proof}

\subsection{Generic spans and generic height theorem}
In this subsection, we will derive the first results on generic ideals. We will derive an statement about spans of generic polynomials, and generic versions of Krull's principal ideal and height theorems which will be the main tool in controlling the structure of generic ideals. This has immediate applications for the cumulant comparison problem.

Now we present the first result which can be easily formulated in terms of genericity:

\begin{Prop}\label{Prop:GenVec}
Let $P$ be an algebraic property such that the polynomials with property $P$ form a vector space $V$. Let $f_1,\dots, f_m\in \C[X_1,\dots X_D]$ be generic polynomials satisfying $P.$ Then
$$\rk \lspan (f_1,\dots, f_m)=\min (m, \dim V).$$
\end{Prop}
\begin{proof}
It suffices to prove: if $i\le M,$ then $f_i$ is linearly independent from $f_1,\dots f_{i-1}$ with probability one. Assuming the contrary would mean that for some $i$, we have
$$f_i=\sum_{k=0}^{i-1}f_k c_k\quad\mbox{for some}\; c_k\in \CC,$$
thus giving several equations on the coefficients of $f_i.$ But these are fulfilled with probability zero by the genericity assumption, so the claim follows.
\end{proof}

This may be seen as a straightforward generalization of the statement: the span of $n$ generic points in $\C^D$ has dimension $\min (n,D).$

We now proceed to another nontrivial result which will now allow us to formulate a generic version of Krull's principal ideal theorem:

\begin{Prop}\label{Prop:NoZero}
Let $Z\subseteq \C^D$ be a non-empty algebraic set, let $f\in \C[X_1,\dots X_D]$ generic. Then $f$ is no zero divisor in $\calO(Z)=\C[X_1,\dots X_D]/\Id(Z).$
\end{Prop}
\begin{proof}
We claim: being a zero divisor in $\calO(Z)$ is an irreducible algebraic property. We will prove that the zero divisors in $\calO(Z)$ form a linear subspace of $\calM_k,$ and linear spaces are irreducible.\\

For this, one checks that sums and scalar multiples of zero divisors are also zero divisors: if $g_1,g_2$ are zero divisors, there must exist $h_1,h_2$ such that $g_1h_1=g_2h_2=0.$ Now for any $\alpha\in \C,$ we have that
$$(g_1+\alpha g_2) (h_1h_2)=(g_1h_1)h_2 + (g_2h_2)\alpha h_1= 0.$$
This proves that $(g_1+\alpha g_2)$ is also a zero divisor, proving that the zero divisors form a linear subspace and thus an irreducible algebraic property.

To apply the genericity assumption to argue that this event occurs with probability zero, we must exclude the possibility that being a zero divisor is trivial, i.e.~always the case. This is equivalent to proving that the linear subspace has positive codimension, which is true if and only if there exists a non-zero divisor in $\calO(Z).$ But a non-zero divisor always exists since we have assumed $Z$ is non-empty: thus $\Id(Z)$ is a proper ideal, and $\calO(Z)$ contains $\C,$ which contains a non-zero divisor, e.g.~the one.

So by the genericity assumption, the event that $f$ is a zero divisor occurs with probability zero, i.e.~a generic $f$ is not a zero divisor. Note that this does not depend on the degree of $f.$
\end{proof}
Note that this result is already known, compare Conjecture B in \citep{Par10}.

A straightforward generalization using the same proof technique is given by the following
\begin{Cor}\label{Cor:NoZero}
Let $\calI\idof \C[X_1,\dots, X_D]$, let $P$ be a non-trivial algebraic property. Let $f\in \C[X_1,\dots X_D]$ be a generic polynomial with property $P$. If one can write $f=f'+c$, where $f'$ is a generic polynomial subject to some property $P'$, and $c$ is a generic constant, then $f$ is no zero divisor in $\C[X_1,\dots, X_D]/\calI.$
\end{Cor}
\begin{proof}
First note that $f$ is a zero divisor in $\C[X_1,\dots, X_D]/\calI$ if and only if $f$ is a zero divisor in $\C[X_1,\dots, X_D]/\sqrt{\calI}.$ This allows us to reduce to the case that $\calI=\Id (Z)$ for some algebraic set $Z\subseteq \C^D.$

Now, as in the proof of Proposition~\ref{Prop:NoZero}, we see that being a zero divisor in $\calO(Z)$ is an irreducible algebraic property and corresponds to a linear subspace of $\calM_k$, where $k=\deg f.$ The zero divisors with property $P$ are thus contained in this linear subspace. Now let $f$ be generic with property $P$ as above. By assumption, we may write $f=f'+c.$ But $c$ is (generically) no zero divisor, so $f$ is also not a zero divisor, since the zero divisors form a linear subspace of $\calM_k.$ Thus $f$ is no zero divisor. This proves the claim.
\end{proof}

Note that Proposition~\ref{Prop:NoZero} is actually a special case of Corollary~\ref{Cor:NoZero}, since we can write any generic polynomial $f$ as $f'+c$, where $f'$ is generic of the same degree, and $c$ is a generic constant.

The major tool to deal with the dimension of generic intersections is Krull's principal ideal theorem:
\begin{Thm}[Krull's principal ideal theorem]\label{Thm:KrullPI}
Let $R$ be a commutative ring with unit, let $f\in R$ be non-zero and non-invertible. Then
$$\htid \langle f\rangle\le 1,$$
with equality if and only if $f$ is not a zero divisor in $R$.
\end{Thm}
The reader unfamiliar with height theory may take
$$\htid \calI = \codim \VS(\calI)$$
as the definition for the height of an ideal (caveat: codimension has to be taken in $R$).

Reformulated geometrically for our situation, Krull's principal ideal theorem implies:
\begin{Cor}\label{Cor:KrullPI-geom}
Let $Z$ be a non-empty algebraic set in $\C^D.$Then
$$\codim (Z\cap \VS(f))\le \codim Z+1.$$
\end{Cor}
\begin{proof}
Apply Krull's principal ideal theorem to the ring $R=\calO(Z)=\C [X_1,\dots, X_D]/\Id(Z).$
\end{proof}
Together with Proposition~\ref{Prop:NoZero}, one gets a generic version of Krull's principal ideal theorem:

\begin{Thm}[Generic principal ideal theorem]\label{Thm:KrullGenPI}
Let $Z$ be a non-empty algebraic set in $\C^D$, let $R=\calO (Z),$  and let $f\in \C[X_1,\dots, X_D]$ be generic. Then we have
$$\htid \langle f\rangle = 1.$$
\end{Thm}
In its geometric formulation, we obtain the following result.
\begin{Cor}\label{Cor:KrullPI}
Consider an algebraic set $Z\subseteq \C^D,$ and the algebraic set $\VS(f)$ for some generic $f\in \C [X_1,\dots, X_D].$
Then
$$\codim (Z\cap  \VS (f))=\min (\codim Z + 1,\; D+1).$$
\end{Cor}
\begin{proof}
This is just a direct reformulation of Theorem~\ref{Thm:KrullGenPI} in the vein of Corollary~\ref{Cor:KrullPI-geom}. The only additional thing that has to be checked is the case where $\codim Z = D+1,$ which means that $Z$ is the empty set. In this case, the equality is straightforward.
\end{proof}

The generic version of the principal ideal theorem straightforwardly generalizes to a generic version of Krull's height theorem. We first mention the original version:
\begin{Thm}[Krull's height theorem]\label{Thm:KrullHt}
Let $R$ be a commutative ring with unit, let $\calI=\langle f_1,\dots, f_m \rangle \idof R$ be an ideal. Then
$$\htid \calI \le m,$$
with equality if and only if $f_1,\dots, f_m$ is an $R$-regular sequence, i.e.~$f_i$ is not invertible and not a zero divisor in the ring $R/\langle f_1,\dots, f_{i-1}\rangle$ for all $i$.
\end{Thm}
The generic version can be derived directly from the generic principal ideal theorem:

\begin{Thm}[Generic height theorem]\label{Thm:KrullGenHt}
Let $Z$ be an algebraic set in $\C^D,$ let $\calI=\langle f_1,\dots, f_m\rangle$ be a generic ideal in $\C [X_1,\dots, X_D].$ Then
$$\htid (\Id(Z)+\calI) = \min (\codim Z + m,\; D+1).$$
\end{Thm}
\begin{proof}
We will write $R=\calO(Z)$ for abbreviation.

First assume $m\le D+1-\codim Z.$ It suffices to show that $f_1,\dots, f_m$ forms an $R$-regular sequence, then apply Krull's height theorem. In Proposition~\ref{Prop:NoZero}, we have proved that $f_i$ is not a zero divisor in the ring
$\calO(Z\cap\VS(f_1,\dots, f_{i-1}))$ (note that the latter ring is nonzero by Krull's height theorem). By Hilbert's Nullstellensatz, this is the same as the ring $R/\sqrt{\langle f_1,\dots, f_{i-1}\rangle}.$ But by the definition of radical, this implies that $f_i$ is no zero divisor in the ring $R/\langle f_1,\dots, f_{i-1}\rangle,$ since if $f_i\cdot h=0$ in the first ring, we have
$$(f_i\cdot h)^N=f_i\cdot (f_i^{N-1}h^N)=0$$
in the second. Thus the $f_i$ form an $R$-regular sequence, proving the theorem for the case $m\le D+1-\codim Z.$

If now $m> k:=D+1-\codim Z,$ the above reasoning shows that the radical of $\Id(Z)+\langle f_1,\dots, f_k\rangle$ is the module $\langle 1\rangle,$ which means that those are equal. Thus
$$\Id(Z)+\langle f_1,\dots, f_k\rangle=\Id(Z)+\langle f_1,\dots, f_m\rangle=\langle 1\rangle,$$
proving the theorem.

Note that we could have proved the generic height theorem also directly from the generic principal ideal theorem by induction.
\end{proof}

Again, we give the geometric interpretation of Krull's height theorem:

\begin{Cor}\label{Cor:genint}
Let $Z_1$ be an algebraic set in $\C^D$, let $Z_2$ be a generic algebraic set in $\C^D$. Then one has
$$\codim (Z_1\cap Z_2)=\min (\codim Z_1+\codim Z_2,\; D+1).$$
\end{Cor}
\begin{proof}
This follows directly from two applications of the generic height theorem~\ref{Thm:KrullGenHt}: first for $Z=\C^D$ and $Z_2=\VS(\calI)$, showing that $\codim Z_2$ is equal to the number $m$ of generators of $\calI;$ then, for $Z=Z_1$ and $Z_2=\VS(\calI),$ and substituting $m=\codim Z_2.$
\end{proof}

We can now immediately formulate a homogenous version of Proposition~\ref{Cor:genint}:

\begin{Cor}\label{Cor:genintproj}
Let $Z_1$ be a homogenous algebraic set in $\C^D$, let $Z_2$ be a generic homogenous algebraic set in $\C^D$. Then one has
$$\codim (Z_1\cap Z_2)=\min (\codim Z_1+\codim Z_2,\; D).$$
\end{Cor}
\begin{proof}
Note that homogenization and dehomogenization of a non-empty algebraic set do not change its codimension, and homogenous algebraic sets always contain the origin. Also, one has to note that by Lemma~\ref{Lem:dehom}, the dehomogenization of $Z_2$ is a generic algebraic set in $\C^{D-1}.$
\end{proof}

Finally, using Corollary~\ref{Cor:NoZero}, we want to give a more technical variant of the generic height theorem, which will be of use in later proofs. First, we introduce some abbreviating notations:
\begin{Def}
Let $f\in \C[X_1,\dots X_D]$ be a generic polynomial with property $P$. If one can write $f=f'+c$, where $f'$ is a generic polynomial subject to some property $P'$, and $c$ is a generic constant, we say that $f$ has {\it independent constant term}. If $c$ is generic and independent with respect to some collection of generic objects, we say that $f$ has independent constant term with respect to that collection.
\end{Def}
In this terminology, Corollary~\ref{Cor:NoZero} rephrases as: a generic polynomial with independent constant term is no zero divisor. Using this, we can now formulate the corresponding variant of the generic height theorem:

\begin{Lem}\label{Lem:KrullGenHt}
Let $Z$ be an algebraic set in $\C^D.$ Let $f_1,\dots, f_m\in\C[X_1,\dots, X_D]$ be generic, possibly subject to some algebraic properties, such that $f_i$ has independent constant term with respect to $Z$ and $f_1,\dots, f_{i-1}.$ Then
$$\htid (\Id(Z)+\calI) = \min (\codim Z + m,\; D+1).$$
\end{Lem}
\begin{proof}
Using Corollary~\ref{Cor:NoZero}, one obtains that $f_i$ is no zero divisor modulo $\Id(Z)+\langle f_1,\dots, f_{i+1}\rangle.$ Using Krull's height theorem yields the claim.
\end{proof}

\subsection{Generic ideals}
The generic height theorem~\ref{Thm:KrullGenHt} has allowed us to make statements about the structure of ideals generated by generic elements without constraints. However, the ideal $\calI$ in our the cumulant comparison problem is generic subject to constraints: namely, its generators are contained in a prescribed ideal, and they are homogenous. In this subsection, we will use the theory developed so far to study generic ideals and generic ideals subject to some algebraic properties, e.g.~generic ideals contained in other ideals. We will use these results to derive an identifiability result on the marginalization problem which has been derived already less rigourously in the supplementary material of \citep{PRL:SSA:2009} for the special case of Stationary Subspace Analysis.

\begin{Prop}\label{Prop:dehom-rad-generic}
Let $\fraks \idof \C [X_1,\dots, X_D]$ be an ideal, having an H-basis $g_1,\dots, g_n$. Let
$$\calI=\langle f_1,\dots, f_m\rangle,\quad m\ge \max(D+1, n)$$
with generic $f_i\in \fraks$ such that
$$\deg f_i\ge \max_j \left(\deg g_j\right)\quad \mbox{for all}\; 1\le i\le m.$$
Then $\calI=\fraks.$
\end{Prop}
\begin{proof}
First note that since the $g_i$ form a degree-first Groebner basis, a generic $f\in \fraks$ is of the form
$$f=\sum_{k=1}^n g_kh_k\quad\mbox{with generic}\;h_k,$$
where the degrees of the $h_k$ are appropriately chosen, i.e. $\deg h_k\le \deg f - \deg g_k$.

So we may write
$$f_i=\sum_{k=1}^n g_kh_{ki}\quad\mbox{with generic}\;h_{ki},$$
where the $h_{ki}$ are generic with appropriate degrees, and independently chosen. We may also assume that the $f_i$ are ordered increasingly by degree.\\

To prove the statement, it suffices to show that $g_j\in \calI$ for all $j$. Now the height theorem~\ref{Thm:KrullGenHt} implies that
$$\langle h_{11},\dots h_{1m}\rangle=\langle 1\rangle,$$
since the $h_{ki}$ were independently generic, and $m\ge D+1.$ In particular, there exist polynomials $s_1,\dots, s_m$ such that
$$\sum_{i=1}^m s_i h_{1i}=1.$$
Thus we have that
\begin{align*}
\sum_{i=1}^m s_i f_i = \sum_{i=1}^m s_i \sum_{k=1}^n g_kh_{ki}= \sum_{k=1}^n g_k \sum_{i=1}^m s_ih_{ki}\\
=g_1+ \sum_{k=2}^n g_k \sum_{i=1}^m s_ih_{ki}=:g_1+ \sum_{k=2}^n g_k h'_k.
\end{align*}
Subtracting a suitable multiple of this element from the $f_1,\dots, f_m,$ we obtain
$$f'_i=\sum_{k=2}^n g_k(h_{ki}-h_{1i}h'_k)=:\sum_{k=2}^n g_k h'_{ki}.$$
We may now consider $h_{1i}h'_k$ as fixed, while the $h_{ki}$ are generic. In particular, the $h'_{ki}$ have independent constant term, and using Lemma~\ref{Lem:KrullGenHt}, we may conclude that
$$\langle h'_{21},\dots, h'_{2m} \rangle=\langle 1\rangle,$$
allowing us to find an element of the form
$$g_2+\sum_{k=3}^n g_k \cdot\dots$$
in $\calI$. Iterating this strategy by repeatedly applying Lemma~\ref{Lem:KrullGenHt}, we see that $g_k$ is contained in $\calI,$ because the ideals $\calI$ and $\fraks$ have same height. Since the numbering for the $g_j$ was arbitrary, we have proved that $g_j\in \calI$, and thus the proposition.
\end{proof}
The following example shows that we may not take the degrees of the $f_i$ completely arbitrary in the proposition, i.e.~the condition on the degrees is necessary:
\begin{Ex}\rm
Keep the notations of Proposition~\ref{Prop:dehom-rad-generic}. Let $\fraks=\langle X_2-X_1^2, X_3\rangle,$ and $f_i\in \fraks$ generic of degree one. Then
$$\langle f_1,\dots, f_m\rangle = \langle X_3\rangle.$$
This example can be generalized to yield arbitrarily bad results if the condition on the degrees is not fulfilled.

However note that when $\fraks$ is generated by linear forms, as in the marginalization problem, the condition on the degrees vanishes.
\end{Ex}

We may use Proposition~\ref{Prop:dehom-rad-generic} also in another way to derive a more detailed version of the generic height theorem for constrained ideals:
\begin{Prop}\label{Prop:KrullHt-algset}
Let $V$ be a fixed $d$-codimensional algebraic set in $\C^D.$ Assume that there exist $d$ generators $g_1,\dots, g_d$ for $\Id(V).$
Let $f_1,\dots, f_m$ be generic forms in $\Id (V)$ such that $\deg f_i\ge \deg g_i$ for $1\le i\le \min (m,d)$. Then we can write $\VS (f_1,\dots, f_m)=V\cup U$ with $U$ an algebraic set of
$$\codim U\ge\min (m,\;D+1),$$
the equality being strict for $m < \codim V.$
\end{Prop}
\begin{proof}
If $m\ge D+1$, this is just a direct consequence of Proposition~\ref{Prop:dehom-rad-generic}.

First assume $m = d.$ Consider the image of the situation modulo $X_{m},\dots, X_D.$ This corresponds to looking at the situation
$$\VS (f_1,\dots, f_m)\cap H\subseteq H\cong \C^{m-1},$$
where $H$ is the linear subspace given by $X_m=\dots = X_D=0.$ Since the coordinate system was generic, the images of the $f_i$ will be generic, and we have by Proposition~\ref{Prop:dehom-rad-generic} that $\VS (f_1,\dots, f_m)\cap H = V\cap H.$ Also, the $H$ can be regarded as a generic linear subspace, thus by Corollary~\ref{Cor:genint}, we see that $\VS (f_1,\dots, f_m)$ consists of $V$ and possibly components of equal or higher codimension. This proves the claim for $m = \codim V.$

Now we prove the case $m\ge d.$ We may assume that $m=D+1$ and then prove the statement for the sets
$\VS (f_1,\dots, f_i), d\le i\le m.$ By the Lasker-Noether-Theorem, we may write
$$\VS (f_1,\dots, f_d)= V \cup Z_1 \cup\dots \cup Z_N$$
for finitely many irreducible components $Z_j$ with $\codim Z_j\ge \codim V.$ Proposition~\ref{Prop:dehom-rad-generic} now states that $$\VS (f_1,\dots, f_m)=V.$$
For $i\ge d,$ write now
$$Z_{ji}=Z_j\cap \VS (f_1,\dots, f_i)= Z_j\cap \VS (f_{d+1},\dots, f_i).$$
With this, we have the equalities
\begin{align*}
\VS (f_1,\dots, f_i)&= \VS (f_1,\dots, f_d)\cap \VS (f_{d+1},\dots, f_i)\\
&= V \cup (Z_1\cap \VS (f_{d+1},\dots, f_i))\cup\dots \cup (Z_N\cap \VS (f_{d+1},\dots, f_i))\\
&= V\cup Z_{1i}\cup\dots\cup Z_{Ni}.
\end{align*}
for $i\ge d.$ Thus, reformulated, Proposition~\ref{Prop:dehom-rad-generic} states that $Z_{jm}=\varnothing$ for any $j$. We can now infer by Krull's principal ideal theorem~\ref{Thm:KrullPI} that
$$\codim Z_{ji}\le \codim Z_{j,i-1}+1$$
for any $i,j$. But since $\codim Z_{jm}=D+1,$ and $\codim Z_{jd}\ge d,$ we thus may infer that $\codim Z_{ji}\ge i$ for any $d\le i\le m.$ Thus we may write
$$\VS (f_1,\dots, f_i)=V\cup U\quad\mbox{with}\;U=Z_{1i}\cup\dots\cup Z_{Ni}$$
with $\codim U\ge i,$ which proves the claim for $m\ge \codim V.$

The case $m < \codim V$ can be proved again similarly by Krull's principal ideal theorem~\ref{Thm:KrullPI}: it states that the codimension of $\VS (f_1,\dots, f_i)$ increases at most by one with each $i$, and we have seen above that it is equal to $\codim V$ for $i=\codim V.$ Thus the codimension of $\VS (f_1,\dots, f_i)$ must have been $i$ for every $i\le \codim V.$ This yields the claim.
\end{proof}
Note that depending on $V$ and the degrees of the $f_i,$ it may happen that even in the generic case, the equality in Proposition~\ref{Prop:KrullHt-algset} is not strict for $m\ge \codim V$:
\begin{Ex}\rm
Let $V$ be a generic linear subspace of dimension $d$ in $\C^D,$ let $f_1,\dots, f_m\in \Id(V)$ be generic with degree one. Then
$\VS (f_1,\dots, f_m)$ is a generic linear subspace of dimension $\max (D-m, d)$ containing $V.$ In particular, if $m\ge D-d,$ then $\VS (f_1,\dots, f_m)=V.$ In this example, $U= \VS(f_1,\dots, f_m)$, if $m < \codim V,$ with codimension $m$, and $U=\varnothing$, if $m\ge \codim V,$ with codimension $D+1.$

Similarly, one may construct generic examples with arbitrary behavior for $\codim U$ when $m\ge \codim V,$ by choosing $V$ and the degrees of $f_i$ appropriately.
\end{Ex}

Similarly as in the geometric version for the height theorem, we may derive the following geometric interpretation of this result:
\begin{Cor}
Let $V\subseteq Z_1$ be fixed algebraic sets in $\C^D$. Let $Z_2$ be a generic algebraic set in $\C^D$ containing $V.$ Then
$$\codim (Z_1\cap Z_2 \setminus V)\ge \min (\codim (Z_1 \setminus V) + \codim (Z_2 \setminus V),\; D+1).$$
\end{Cor}
Informally, we have derived a height theorem type result for algebraic sets under the constraint that they contain another prescribed algebraic set $V$. \\

We also want to give a homogenous version of Proposition~\ref{Prop:KrullHt-algset}, since the ideals in the paper are generated by homogenous forms:
\begin{Cor}\label{Cor:KrullHt-Hom}
Let $V$ be a fixed homogenous algebraic set in $\C^D$.
Let $f_1,\dots, f_m$ be generic homogenous forms in $\Id (V),$ satisfying the degree condition as in Proposition \ref{Prop:KrullHt-algset}. Then $\VS (f_1,\dots, f_m)=V+ U$ with $U$ an algebraic set fulfilling
$$\codim U\ge \min (m,\;D).$$
In particular, if $m> D,$ then $\VS (f_1,\dots, f_m)=V.$
Also, the maximal dimensional part of $\VS (f_1,\dots, f_m)$ equals $V$ if and only if $m > D- \dim  V.$
\end{Cor}
\begin{proof}
This follows immediately by dehomogenizing, applying Proposition~\ref{Prop:KrullHt-algset}, and homogenizing again.
\end{proof}
From this Corollary, we now can directly derive a statement on the necessary number of epochs for the identifiability of the projection making several random variables appear identical. For the convenience of the reader, we recall the setting and then explain what identifiability means. The problem we consider in the main part of the paper can be described as follows:

\begin{Prob}
Let $X_1,\dots, X_{m}$ be smooth random variables, let
$$q_i=[T_1,\dots, T_D]\circ \left(\kappa_2(X_i)-\kappa_2(X_{m})\right),\;1\le i\le {m-1}$$
and
$$f_i=[T_1,\dots, T_D]\circ\left( \kappa_1(X_i)-\kappa_1(X_{m})\right),\;1\le i\le {m-1}$$
be the corresponding polynomials in the formal variables $T_1,\dots, T_D$.
What can one say about the set
$$S'=\VS ( q_1, \ldots, q_{m-1}, f_1,\dots, f_{m-1} ).$$
\end{Prob}
If there is a linear subspace $S$ on which the cumulants agree, then the $q_i, f_i$ vanish on $S$. If we assume that this happens generically, the problem reformulates to
\begin{Prob}
Let $S$ be a $d$-dimensional linear subspace of $\C^D$, let $\fraks=\Id(S),$ and let $f_1,\dots, f_{N}$ be generic homogenous quadratic or linear polynomials in $\fraks$. How does $S'=\VS ( f_1,\dots, f_{N} )$ relate to $S$?.
\end{Prob}

Before giving bounds on the identifiability, we first begin with a direct consequence of Corollary~\ref{Cor:KrullHt-Hom}:
\begin{Rem}\label{Rem:cumpoly}
The highest dimensional part of $S'=\VS (f_1,\dots, f_N)$ is $S$ if and only if
$$ N > D - d.$$
\end{Rem}
For this, remark that $\Id(S)$ is generated in degree one, and thus the degree condition in Corollary \ref{Cor:KrullHt-Hom} becomes empty.

We can now also get an identifiability result for $S$:

\begin{Prop}\label{Prop:ident_ex}
Let $f_1,\dots, f_{N}$ be generic homogenous polynomials of degree one or two, vanishing on a linear space $S$ of dimension $d>0$. Then $S$ is identifiable from the $f_i$ alone if
$$N \ge D-d+1.$$
Moreover, if all $f_i$ are quadrics, then $S$ is identifiable from the $f_i$ alone only if
$$N\ge 2.$$
\end{Prop}
\begin{proof}
Note that the $f_1,\dots, f_N$ are generic polynomials contained in $\fraks:=\Id(S)$.

First assume $N \ge D-d+1.$ We prove that $S$ is identifiable: using Corollary~\ref{Cor:KrullHt-Hom}, one sees now that the common vanishing set of the $f_i$ is $S$ up to possible additional components of dimension $d-1$ or less. I.e.~the radical of the ideal generated by the $f_i$ has a prime decomposition
$$\sqrt{\langle f_1,\dots, f_{N}\rangle} = \fraks\cap \frakp_1\cap \dots \cap\frakp_k,$$
where the $\frakp_i$ are of dimension $d-1$ or less, while $\fraks$ has dimension $d$. So one can use one of the existing algorithms calculating primary decomposition to identify $\fraks$ as the unique component of the highest dimensional part, which proves identifiability if $N \ge D-d+1$.

Now we prove the only if part: assume that $N=1$, i.e.~we have only a single $f_1$.
Since $f_1$ is generic with the property of vanishing on $S$, we have
$$f_1=\sum_{i=1}^{D-d}g_ih_i,$$
where $g_1,\dots, g_{D-d}$ is some homogenous linear generating set for $\Id (S),$ and $h_1,\dots, h_{D-d}$ are generic homogenous linear forms. Thus, the zero set $\VS(f_1)$ also contains the linear space $S'=\VS (h_1,\dots, h_{D-d})$ which is a generic $d$-dimensional linear space in $\mathbb{C}^D$ and thus different from $S$; no algorithm can decide whether $S$ or $S'$ is the correct solution, so $S$ is not identifiable.
\end{proof}

Note that there is no obvious reason for the lower bound $N \ge D-d+1$ given in Proposition~\ref{Prop:ident_ex} to be strict. While it is most probably the best possible bound which is in $D$ and $d$, in general it may happen that $S$ can be reconstructed from the ideal $\langle f_1,\dots, f_{N}\rangle$ directly. The reason for this is that a generic homogenous variety of high enough degree and dimension does not need to contain a linear subspace of fixed dimension $d$ in general.

\newpage

\vskip 0.2in

\bibliography{../bib/algebra,../bib/stable,../bib/ssa_pubs,../bib/ida}

\end{document}